\documentclass[10pt]{article} 
\usepackage[preprint]{tmlr}

\usepackage[utf8]{inputenc} 
\usepackage[T1]{fontenc}    
\usepackage{hyperref}       
\usepackage{url}            
\usepackage{booktabs}       
\usepackage{amsfonts}       
\usepackage{nicefrac}       
\usepackage{microtype}      
\usepackage{xcolor}         
\usepackage{amsmath}
\usepackage{amssymb}
\usepackage{mathtools}
\usepackage{physics}
\mathtoolsset{showonlyrefs=true}
\usepackage[table]{xcolor}
\usepackage{enumitem}

\usepackage{amsthm}
\usepackage{graphicx}
\usepackage{sidecap}
\usepackage{commath}
\usepackage{float}
\graphicspath{{./FIG/}}
\DeclareGraphicsExtensions{.eps,.png,.jpg,.jpeg,.bmp,.gif,.pdf}
\usepackage{wrapfig}
\usepackage{caption}
\usepackage[makeroom]{cancel}
\usepackage{tocloft}

\usepackage{comment}

\theoremstyle{plain}
\newtheorem{theorem}{Theorem}[section]
\newtheorem{proposition}[theorem]{Proposition}

\theoremstyle{remark}
\newtheorem{remark}[theorem]{Remark}

\title{Learning infinite context windows in recurrent architectures via spatial neural computing}

\author{\name Aleix Salvador-Pomarol \email asalvadorpomarol@ucsd.edu \\
  \addr UC San Diego
  \AND
  \name Arthur N. Montanari \email arthur.montanari@northwestern.edu \\
  \addr Northwestern University
  \AND
  \name Earl K. Miller \email ekmiller@mit.edu \\
  \addr Massachusetts Institute of Technology
  \AND
  \name Adilson E. Motter \email motter@northwestern.edu \\
  \addr Northwestern University
  \AND
  \name Jorge Cort\'es \email jcortes@ucsd.edu \\
  \addr UC San Diego}

\begin{document}

\maketitle

\begin{abstract}
  Recurrent neural networks (RNNs) offer linear-time scaling with sequence length while requiring only constant memory, yet they struggle to capture long-range dependencies due to vanishing gradients and limited receptive fields. To address these limitations, we introduce a second-order recurrent model in which the standard neuron-to-neuron communication is replaced by a spatially evolving field governed by (discretized) partial differential equations. Drawing inspiration from the role of cortical waves in brain computation, this mechanism allows structured spatiotemporal patterns to serve as an implicit, high-capacity memory. We show that the resulting model is equivalent to a structured infinite-order RNN in which the current state depends explicitly on its entire history of past states, yielding an effectively unbounded receptive field with a fixed number of parameters. We further derive constructive conditions to ensure marginal stability, constraining the gradient spectrum on the unit circle and thereby eliminating vanishing and exploding gradients. Empirically, the proposed architecture outperforms other recurrent models on long-horizon benchmarks while using substantially fewer parameters, demonstrating that spatial dynamics can effectively bridge the gap between efficient inference and long-term memory.
%
\end{abstract}

\section{Introduction}
For decades, recurrent neural networks (RNNs) have served as the cornerstone of sequence modeling. These models rely on hidden states that are recurrently updated to process sequential inputs, acting as a dynamical memory of past information. In contrast to stateless architectures, RNNs offer the theoretical promise of compressing an unbounded context into a fixed-size state, resulting in $O(1)$ memory cost with respect to the sequence length $T$ and $O(T)$ computational complexity.
However, the practical efficacy of classical RNNs is severely constrained by the vanishing and exploding gradient problem~\citep{bengio1994learning}. From a dynamical systems perspective, this issue stems from the spectrum of the recurrence matrix. When its eigenvalues are stable ($|\lambda|<1$) or unstable ($|\lambda|>1$), perturbations respectively decay or grow exponentially, leading to a loss of long-term dependencies during training. While unitary and orthogonal RNNs~\citep{arjovsky2016unitary,ceni2025random,henaff2016recurrent, kerg2019non,le2015simple} can mitigate this problem by constraining the eigenvalues to be \textit{marginally stable} ($|\lambda|=1$), the nonlinearities of the activation functions perturb the effective dynamics away from the unit circle, reintroducing vanishing/exploding gradients. Moreover, such constraints limit the model expressivity and convergence speed~\citep{kerg2019non,ribeiro2020beyond}.
To bypass these spectral issues, architectures such as the long short-term memory (LSTM)~\citep{hochreiter1997long} and gated recurrent unit (GRU)~\citep{cho2014learning} introduce gating mechanisms with linear gradient updates.
These additive gates prevent gradients from exponentially decaying due to repetitive multiplicative operations, but they also come at the cost of increased architecture complexity and parameter count, leading to higher computational overhead \citep{chung2014empirical}, limited parallelization \citep{li2018independently}, and implicit timescale biases \citep{neil2016phased}.

Recently, the transformer architecture has largely supplanted RNNs due to its ability to model long-range dependencies via self-attention.
However, the transformer’s supremacy comes at a steep cost:  the key-value memory grows linearly, $\mathcal O(T)$, during inference and the attention mechanism scales quadratically, $\mathcal O(T^2)$ \citep{kwon2023efficient,vaswani2017attention}. These scaling properties effectively reintroduce the ``von Neumann gap'' in modern artificial intelligence (AI) as a significant portion of inference time and energy is spent on memory access rather than computation \citep{dao2022flashattention}. 
This computational bottleneck has renewed interest in dynamical-systems-based architectures \citep{sieber2024understanding}, such as state-space models (SSMs)~\citep{gu2024mamba,gu2021efficiently, somvanshi2025s4} and oscillator networks~\citep{effenberger2025functional,keller2023neural,keller2023traveling,rusch2024oscillatory,rusch2022graph}. 
Although these approaches improve memory-computation trade-offs, they typically rely on dissipative dynamics. As a result, they do not provide mechanisms to maintain nonvanishing gradients over \textit{arbitrarily long horizons} \citep{gu2024mamba,rusch2020coupled}. 
Moreover, controlling memory-retention and forgetting mechanisms in these models remains challenging, often requiring increased parameterization or complex hybrid architectures to achieve sufficient expressivity~\citep{karuvally2025bridging,wang2025comprehensive}. 

Taken together, these limitations raise a fundamental question: \textit{What structural features allow recurrent architectures to maintain nonvanishing gradient modes over long horizons despite input-dependent nonlinear perturbations?}  Through a spectral analysis of the gradient dynamics, we show that the answer lies in \textit{how} the nonlinearity enters the gradient equations, rather than in the spectrum of the recurrence matrices. In conventional first-order recurrent architectures (e.g., RNNs), these perturbations directly affect the coefficients that control modal decay, making robust marginal stability impossible regardless of the parameterization. In contrast, a second-order recurrent architecture can have additional, input-independent coefficients that selectively fix the magnitude of gradient modes to enforce marginal stability while allowing inputs to perturb the remaining dynamics. Such a separation arises naturally, for example, in wave propagation dynamics, since its inertial term is determined by the medium itself rather than by the input signal driving it. This connection suggests that wave dynamics may be potentially used as a natural computational substrate for memory retention\textemdash an idea further motivated by increasing evidence on the role of cortical traveling waves in encoding, maintaining, and transforming information through distributed spatiotemporal activity patterns~\citep{benigno2023waves,davis2021spontaneous,ermentrout2001traveling,singer2021recurrent}.

In this paper, we propose a \textit{spatial computing model}, called SpatialRNN, in which recurrent neuron-to-neuron communication is mediated through a spatially evolving medium governed by (discretized) partial differential equations (PDEs). This mechanism allows structured spatiotemporal patterns to be encoded and processed directly in the state of the medium, providing a dynamical substrate for retaining long-term temporal dependencies.
In fact, we rigorously prove that the SpatialRNN is mathematically equivalent to an RNN in which the current state $h_t$ depends explicitly on all its past states $\{h_{t-1},h_{t-2},h_{t-3},\ldots\}$. The SpatialRNN therefore provides an implicit and parameter-efficient realization of an unbounded receptive field. Crucially, the second-order gradient recurrence induced by the spatial medium can be designed to support nonvanishing gradient modes over arbitrarily long time horizons, ensuring robust marginal stability. To achieve this, we derive constructive conditions that constrain the (frozen-time) gradient modes on circles of prescribed radii for every admissible nonlinear perturbation, providing explicit control over memory retention (unit radius) and forgetting (smaller radii) through tunable decay rates.  
Our empirical results show that the gradient norms of trained SpatialRNNs remain nearly constant over long input trajectories and that the SpatialRNN achieves competitive performance on long-horizon benchmarks while often requiring considerably fewer parameters, highlighting the role of spatial computing as an efficient computational paradigm. 

\section{Spatial Computing Model}\label{sec:model}

We propose the SpatialRNN, a recurrent architecture designed to store long-term dependencies in sequential data by fundamentally changing how neurons communicate among themselves. In standard RNNs, information is propagated through a learned recurrence matrix $W_h$, resulting in instantaneous, all-to-all interactions between hidden (neuron) states $h_t \in \mathbb R^n$. Instead, the SpatialRNN replaces this direct neuron-to-neuron communication with a dynamic communication medium $\psi(x,t)$ through which information propagates over space and time. Each neuron interacts locally with this medium, which evolves according to a PDE. This induces structured, wave-like information propagation, allowing information to persist within the evolving medium rather than being repeatedly compressed into a finite-dimensional hidden state.

Formally, the medium evolves according to the second-order linear PDE:
\begin{equation}\label{eq:pde}
    \rho(x)\frac{\partial^2 \psi}{\partial t^2}+ \gamma(x) \frac{\partial \psi}{\partial t} = \mathcal L (\psi) + \beta(h(t)),
\end{equation}
where $x\in\mathbb R^2$ denotes the 2-dimensional spatial coordinate and $\psi(x,t)$ is the state of the communication medium. Here, $\mathcal L$ is a spatial linear operator modeling local interactions in the medium (e.g., $\mathcal L$ could be the Laplacian $\Delta = \frac{\partial^2 \psi}{\partial x_1^2} + \frac{\partial^2 \psi}{\partial x_2^2}$), $\rho(x)$ and $\gamma(x)$ are spatially-dependent functions controlling the local propagation and damping dynamics, respectively, and $\beta$ maps the neural activity $h(t)$ onto the spatial domain, acting as a forcing term on the medium $\psi(x,t)$. 

To implement this model in discrete time, we discretize the PDE \eqref{eq:pde}, yielding the following second-order linear difference equation (LDE):
\begin{equation}\label{eq:lde}
    \psi_t = \overline W \psi_{t-1} + \overline{\overline W}\psi_{t-2} + Bh_{t-1},
\end{equation}
where ${\overline W}, \overline{\overline W} \in \mathbb R^{n\times n}$ define the medium dynamics, and $B \in \mathbb R^{n\times n}$ acts as a forcing term driven by the neural state $h_t\in\mathbb R^n$. The SpatialRNN integrates this medium into the recurrent computation structure of an RNN, leading to our proposed model architecture:
\begin{align}
\text{SpatialRNN}\begin{cases}\label{eq:SpatialRNN}
    h_t = \Phi(W_x x_t + W_\psi\psi_t + b_h), \\
    \psi_t = \overline W \psi_{t-1} + \overline{\overline W}\psi_{t-2} + B h_{t-1},\\
    y_t = \text{softmax}(W_y^h h_t + W^{\psi,1}_y\psi_{t} +  W^{\psi,2}_y\psi_{t-1}+ b_y).
    \end{cases}
\end{align}
Here,  $x_t \in \mathbb R^m$ is the time-dependent input signal, $h_t\in\mathbb R^n$ is the hidden neuron state, and $y_t \in \mathbb R^p$ is the output state. The model is parameterized by the (trainable) input-to-state weights $W_x \in \mathbb R^{n \times m}$, medium-to-state weights $W_\psi \in \mathbb R^{n \times n}$, medium weights ${\overline W}, \overline{\overline W} \in \mathbb R^{n\times n}$, output weights $W_y^h, W^{\psi,1}_y, W^{\psi,2}_y \in \mathbb R^{p \times n}$, and biases $b_h \in \mathbb R^n$, $b_y \in \mathbb R^p$. 
The activation function $\Phi:\mathbb R^n\mapsto\mathbb R^n$ is applied elementwise, with $\Phi(z) = (\Phi(z_1),\ldots,\Phi(z_n))^\top$, and can represent any standard nonlinearity (e.g., tanh or ReLU). For classification tasks, a softmax output is taken without loss of generality. For initialization, we set $h_0=\psi_0 =\psi_{-1}= 0$. The medium $\psi_t\in \mathbb R^n$ serves as a \textit{spatiotemporal memory}, allowing information to propagate and interact over time without repeated compression.
Fig.~\ref{fig:cells_and_nwks}a illustrates the spatiotemporal dynamics of the SpatialRNN with a 2-dimensional medium. 

\begin{figure}[t]
    \centering
    \includegraphics[width=\linewidth]{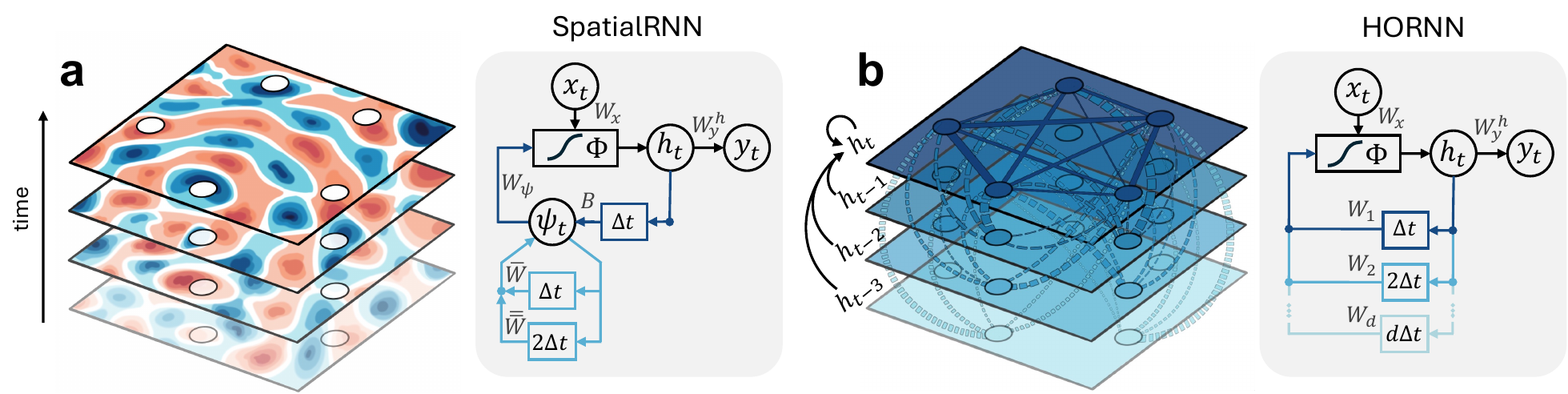}
    \caption{Schematic of the model architectures. 
    (\textbf{a}) \textit{SpatialRNN:} The neuron states $h_t$ (represented by white nodes) are coupled through a shared spatiotemporal communication medium $\psi_t$ (color coded), where information propagates locally over space and time. Each node writes to and reads from the medium, inducing wave-like interactions that encode past activity.
    (\textbf{b}) \textit{HORNN:} Nodes correspond to neuron states $h_t$, and interactions are explicitly represented by edges connecting the current state to a finite set of past states $h_{t-k}$, for $k=1,\ldots,d$. For $d=1$, the model reduces to a standard RNN. 
    Block diagrams of the corresponding discretized computations are shown on the right of each panel.
    }
    \label{fig:cells_and_nwks}
\end{figure}
%
%

\smallskip
\begin{remark}\label{rmk:stormerverlet}
    A PDE \eqref{eq:pde} of interest is the \textit{wave equation}, given by $\frac{\partial^2\psi}{\partial t^2} = c^2\Delta\psi + \beta(h(t))$, where $\Delta$ is the Laplacian operator and $c$ is the wave speed. Applying the temporal discretization $\frac{\partial^2\psi}{\partial t^2} \approx \frac{\psi_{t}-2\psi_{t-1} +\psi_{t-2}}{\Delta t^2}$ and the spatial discretization $\Delta \psi \approx W_\Delta \psi$, with the discrete Laplacian operator $W_\Delta \in \mathbb R^{n \times n}$, yields an LDE of the form~\eqref{eq:lde}. The update matrices are then $\overline W = 2\mathbb I + c^2\Delta t^2 W_\Delta$ and $\overline{\overline W} = -\mathbb I$ (where $\mathbb I$ denotes the identity matrix), corresponding to the classical Störmer-Verlet discretization.
\end{remark}

\subsection{Higher-Order and Standard RNNs}

A straightforward approach to extend the long-term memory of an RNN is to introduce direct dependencies of the current state $h_t$ on multiple time-delayed states $\{h_{t-1},h_{t-2},\ldots h_{t-d}\}$. This concept leads to  \textit{higher-order RNNs} (HORNNs) \citep{soltani2016higher}, which generalize standard RNNs by explicitly incorporating a finite history of past hidden states (Fig.~\ref{fig:cells_and_nwks}b). In this formulation, the update rule is written as a discrete convolution over past states, parameterized by a set of (trainable) recurrence matrices $\mathcal{W} = \{W_k\}_{k=1}^d$, with $W_k \in \mathbb{R}^{n \times n}$:
\begin{align}
       \text{HORNN}\begin{cases}\label{eq:HORNN}
        h_t = \Phi(W_x x_t + \sum_{k=1}^{d}W_k h_{t-k} + b_h), \\
        y_t = \text{softmax}(W_y^h h_t + b_y).
    \end{cases}
\end{align}
In practice, HORNNs exhibit a clear trade-off between memory capacity and parameter efficiency. The order $d\in\mathbb N$ defines the temporal receptive field of the model, but also increases the number of parameters linearly, since each additional lag $k$ introduces a new matrix $W_k$. Moreover, for any finite $d$, dependencies beyond lag $d$ are not modeled directly and must propagate through intermediate states, leading again to exponential decay of long-range memories for all lags beyond $d$ (see results in Sec.~\ref{subsec:hornn_standard_nonrobust} and Fig.~\ref{fig:gradient_plots} for detailed theoretical analysis). For $d=1$, the HORNN model reduces to the \textit{standard RNN}:
\begin{align}
    \text{Standard RNN}\begin{cases}\label{eq:standardRNN}
    h_t = \Phi(W_x x_t + W_h h_{t-1} + b_h), \\
    y_t = \text{softmax}(W_y h_t + b_y).
\end{cases}
\end{align}

\smallskip
\begin{remark}
From a theoretical standpoint, one may consider the limit $d \to \infty$, corresponding to an \textit{$\infty$-order RNN}. In this limit, the recurrent term becomes a full convolution over the entire history: $\sum_{k=1}^{t-1} W_k h_{t-k}$. While the $\infty$-order RNN eliminates the vanishing gradient problem by providing a direct, length-1 path to any past state $h_{t-k}$, it comes at prohibitive cost, requiring $\mathcal{O}(tn^2)$ parameters and $\mathcal{O}(tn)$ memory at step $t$. As a result, the model loses the main edge of state-based architectures.
\end{remark}

%


\subsection{Equivalence between SpatialRNNs and $\infty$-order RNNs}

The effective temporal context window of a finite-order HORNN \eqref{eq:HORNN} is limited by its receptive field $d$. Extending it to an infinite horizon $d \to \infty$ theoretically solves this limitation but becomes computationally intractable for long-sequence tasks due to the prohibitively large parameter and memory requirements.
Our key insight is that the SpatialRNN provides an implicit parameterization of this infinite history. Thus, instead of storing past states or learning a separate matrix $W_k$ for every lag $k$, the SpatialRNN generates the effective convolution kernel $\{W_k\}_{k=1}^{\infty}$ via the dynamics of the spatial medium, as formalized below.

\smallskip
\begin{theorem}[Equivalence between medium state and infinite history]
\label{thm:SpatialRNN_HORNN}
Let $\{h_t\}_{t \ge 1}$ be a sequence of hidden states and let $\psi_t$ be governed by a second-order LDE \eqref{eq:lde} 
with initial conditions $\psi_0 = \psi_{-1} = 0$. Then, the spatial state $\psi_t$ is equivalent to a convolution over the entire history of hidden states:
\begin{equation}
    \psi_t = \sum_{k=1}^{t-1} W_k h_{t-k},
    \label{eq:spatial_convolution}
\end{equation}
where the sequence of weight matrices $\{W_k\}_{k \ge 1}$ are defined recursively by
\begin{equation}\label{eq:recurrence}
W_1 = B, \quad W_2 = \overline W B, \quad \text{and} \quad
    W_k = \overline W W_{k-1} + \overline{\overline W} W_{k-2}, \,\, k\geq 3.
\end{equation}
\end{theorem}
\begin{proof}
See Appendix \ref{apx:proofs}. The result follows by recursively expanding \eqref{eq:lde} and matching the coefficients for each lag $h_{t-k}$. In particular, the sequence of matrices \eqref{eq:recurrence} corresponds to the impulse response~\citep{ogata1995discrete}
%
%
of the LDE \eqref{eq:lde} (i.e., the system response to a unit input at each lag).
\end{proof}

Theorem \ref{thm:SpatialRNN_HORNN} establishes a direct connection between the SpatialRNN and the $\infty$-order RNN, highlighting two fundamental advantages of our approach as $t\rightarrow\infty$.
On \textit{memory efficiency}, while the $\infty$-order RNN requires full history $\{h_k\}_{k=1}^{t-1}$ to compute the next state $h_t$, the SpatialRNN compresses this entire history into the medium states $\psi_t$ and $\psi_{t-1}$, restoring $\mathcal{O}(1)$ memory complexity during inference. On \textit{parameter efficiency}, an $\infty$-order RNN requires infinitely many unique parameters $\{W_k\}_{k=1}^{\infty}$, while the SpatialRNN uses only four fixed-size matrices $\{\overline{W}, \overline{\overline{W}}, B, W_{\psi}\}$ to implicitly encode this infinite sequence of matrices. The SpatialRNN can thus be viewed as a feasible realization of a \textit{structured} $\infty$-order RNN that is both memory and parameter efficient. By virtue of this equivalence, the SpatialRNN has an unbounded receptive field, allowing hidden states at arbitrary lags to influence the current state; whether this influence is retained during training is determined by the gradient dynamics, as we discuss next.

\smallskip
\begin{remark}\label{rmk:receptive_field}
Theorem~\ref{thm:SpatialRNN_HORNN} establishes an \textit{unbounded receptive field}, in the sense that every past state can, in principle, influence the current one. This is distinct from the notion of \textit{long-lived influence}, which requires the sensitivity of the current state to a past state to remain nonvanishing as their temporal separation increases, rather than decay exponentially (see Sec.~\ref{sec:Gradients} for conditions under which the SpatialRNN can ensure long-lived influence). It is also distinct from \textit{information capacity}, since an unbounded receptive field does not imply that a fixed-dimensional, finite-precision state can losslessly retain arbitrary information from an indefinitely long history. 
\end{remark}


\section{Gradient Dynamics and Conditions for Long-Term Memory}\label{sec:Gradients}
%
%


RNNs are typically trained via backpropagation through time (BPTT). For a standard first-order RNN architecture, the gradient of a sequence-to-output loss $\mathcal{L}$ with respect to the weights $W$ can be expressed via the chain rule as 
\begin{equation}
    \frac{\partial \mathcal{L}}{\partial W} = \frac{\partial \mathcal{L}}{\partial y_T} \frac{\partial y_T}{\partial h_{T}} \sum_{k=1}^{T}
\frac{\partial h_{T}}{\partial h_k}
\frac{\partial h_k}{\partial W},
    \label{eq:bptt_sum}
\end{equation}
where $T$ is the sequence length. 
This decomposition makes explicit that the loss gradients are governed by the state-to-state Jacobian matrices $\pdv{h_T}{h_k}$,
%
%
which quantify how perturbations propagate across time. We can recursively define the $m$-step state transition Jacobian matrix as follows:
\begin{equation} \label{eq:gradient_product}
\begin{aligned}
    G_{m}^{t} &:= \frac{\partial h_{t+m}}{\partial h_{t}}
    = \prod_{k=0}^{m-1} \frac{\partial h_{t+k+1}}{\partial h_{t+k}} 
    \quad \implies \quad
    G_{m}^{t} = A_{t+m-1} G_{m-1}^t,
\end{aligned}
\end{equation}
where $A_t := \frac{\partial h_{t+1}}{\partial h_{t}} \in\mathbb R^{n\times n}$ and $G_{1}^{t} = \frac{\partial h_{t+1}}{\partial h_{t}}$. Eq.~\eqref{eq:gradient_product} constitutes a \textit{discrete-time linear time-varying (LTV) dynamical system} because the matrices $A_{t}$ depend on the time-varying input sequence $x_t$.
As the temporal distance $m$ grows, the asymptotic behavior of $\|G_{m}^{t}\|$ determines whether the model can effectively learn long-term dependencies. If $\|G_{m}^{t}\|$ decays to zero (vanishing gradients) or grows unboundedly (exploding gradients), the training fails to capture long-horizon correlations~\citep{bengio1994learning}. Hereafter, we refer to the evolution of $G_m^t$ across time as the gradient dynamics.

To build intuition, we recall that, for a discrete-time linear time-invariant (LTI) system $z_m = Az_{m-1}$ with spectral radius $\rho(A) = \max_{\lambda\in\sigma(A)}|\lambda|$, the state decays exponentially if $\rho(A)<1$ (\textit{vanishing gradients}), grows unboundedly if $\rho(A)>1$ or if nonsimple eigenvalues lie on the unit circle (\textit{exploding gradients}), and remains bounded and nondecaying if $\rho(A)=1$ with all unit-modulus eigenvalues semisimple (\textit{marginal stability)}.
Our goal is to identify the structural conditions under which the gradient dynamics are marginally stable for every admissible input-dependent perturbation, allowing for arbitrarily long memory retention. However, establishing marginal stability conditions for arbitrary LTV systems, such as Eq.~\eqref{eq:gradient_product}, is generally challenging~\citep{sun2008note}.
To circumvent this challenge, we adopt instead a \textit{frozen-time approximation}, in which we hold the instantaneous transition matrix $A_t$ fixed at each reference time $t$ and analyze the spectrum of the resulting LTI system. In contrast, the full LTV gradient dynamics evolve through the product $G^t_m = A_{t+m-1}\cdot\cdot\cdot A_{t+1}A_t$. The frozen-time approximation thus serves as a tractable proxy for practical stability, since the marginal stability of every frozen-time system is a necessary, but not sufficient, condition for the marginal stability of the full LTV system under arbitrary switching of the perturbations. 

In what follows, we show that finite-order HORNNs, including standard RNNs, cannot robustly maintain marginal stability, and therefore inherently fall into regimes of vanishing or exploding gradients (Sec.~\ref{subsec:hornn_standard_nonrobust}). In contrast, we prove that the SpatialRNN, under suitable conditions, allows for robust eigenvalue placement on the unit circle, thereby achieving robust marginal stability of the frozen-time gradient dynamics (Sec.~\ref{subsec:spatialRNN_grads}). 

\subsection{Nonrobustness of RNN and HORNN Gradients}
\label{subsec:hornn_standard_nonrobust}

We now analyze the gradient dynamics of HORNNs. The key issue of these models is the lack of robustness: even when parameters are carefully initialized to have marginally stable modes, small variations induced by inputs and nonlinearities inevitably push the system into vanishing or exploding regimes.
Adopting a frozen-time approximation, we fix a reference time $t$ and examine the gradient dynamics of $G_m^t$. As in Eq.~\eqref{eq:gradient_product}, we characterize the gradient dynamics of HORNNs through the state-to-state Jacobian $G^t_m = \pdv{h_{t+m}}{h_t}$, yielding the following recurrence equation:
\begin{equation} \label{eq:HORNNgrad}
        G_m^t = D_m^t\sum\nolimits_{k=1}^{\min\{m, d\}}W_k G_{m-k}^t, \quad m\geq 1,
\end{equation}
initialized with $G_0^t = \mathbb I$ and $G_m^t = 0$ for $m<0$. Here, $D_m^t$ is a diagonal matrix defined as
\begin{equation}
    D_m^t = \text{diag}\Big(\Phi'\Big(W_x x_{t+m} + \sum\nolimits_{k=1}^{d}W_k h_{t+m-k} + b_h\Big)\Big),
\end{equation}
where $\Phi'$ denotes the derivative of the activation function $\Phi$. In the special case of a standard RNN ($d=1$), Eq.~\eqref{eq:HORNNgrad} reduces to $G_m^t = D_m^t W_h G_{m-1}^t$.

Crucially, for both HORNNs and standard RNNs, each matrix $D_{m}^{t}$ depends on the time-varying input $x_{t+m}$ and can therefore vary over time in a manner that is effectively uncontrolled. However, their entries are constrained by the range of the activation function derivative. This implies that, at each time step, $D_{m}^{t}$ lies in the set
\begin{equation*}
\mathcal D = \{D\in\mathbb R^{n\times n} \, | \, D = \operatorname{diag}(d_1,\ldots,d_n), \,\, d_i\in[\Phi'_{\min}, \Phi'_{\max}]\},
\end{equation*}
where $\Phi'_{\min} = \inf_{z\in\mathbb R} \Phi'(z)$ and $\Phi'_{\max} = \sup_{z\in\mathbb R} \Phi'(z)$ denote the bounds of the function $\Phi'$.
%
To assess robustness, we consider the family of systems obtained by fixing $D_{m}^{t} = D$, $\forall m$, with some constant matrix $D \in \mathcal{D}$. Each such choice corresponds to a frozen-time (LTI) system $G_m^t = D \sum_{k} W_k \, G_{m-k}^t$. A necessary condition for robust marginal stability \citep{liberzon2003switching} is that the frozen-time gradient dynamics remain marginally stable for every $D\in\mathcal D$.

\smallskip
\begin{proposition}[Spectrum of frozen-time HORNN gradients]\label{prop:HORNN_det_0}
For any fixed $D \in \mathcal{D}$, a value $\lambda \in \mathbb{C}$ is an eigenvalue of the HORNN frozen-time gradient dynamics \eqref{eq:HORNNgrad} if and only if
\begin{equation}\label{eq:HORNN_det_0}
        \det\big(\mathbb I - D\hat W(\lambda)\big) = 0, \quad \text{where} \quad \hat W(\lambda) = \sum\nolimits_{k=1}^d  W_k \lambda^{-k}.
    \end{equation}
\end{proposition}

\begin{proof}
    The proof exploits the ${z}$-transform of the linear recurrence system and is given in Appendix~\ref{apx:proofs}
\end{proof}

The characteristic equation \eqref{eq:HORNN_det_0} reveals 
that the eigenvalues $\lambda$ are roots of a polynomial whose coefficients depend linearly on the diagonal entries of $D$. Consequently, variations in $D$ directly perturb the spectrum, preventing robust marginal stability, as we formalize next.

\smallskip
\begin{proposition}[Non-existence of marginally stable HORNNs]
\label{prop:hornn_instability}
Suppose the derivative of the activation function $\Phi$ takes values arbitrarily close to 0. For any HORNN~\eqref{eq:HORNN}, defined by a set of weight matrices $\{W_k\}_{k=1}^d$ with finite $d\in\mathbb N$, there exists a matrix $D \in \mathcal{D}$ such that the gradient dynamics are not marginally stable (i.e., no root $\lambda$ of Eq.~\eqref{eq:HORNN_det_0} satisfies $|\lambda| = 1$).
\end{proposition}

\begin{proof}
    See proof in Appendix \ref{apx:proofs}
\end{proof}

Proposition~\ref{prop:hornn_instability} identifies a structural limitation of finite-order recurrent architectures, since its proof relies only on the fact that $D$ multiplies the entire recurrence operator. The result is therefore independent of the specific parameterization of the recurrence matrices and extends to a broad class of architectures, including orthogonal~\citep{helfrich2018orthogonal,henaff2016recurrent}, nonnormal~\citep{kerg2019non}, and diagonal~\citep{li2018independently} RNNs. As a result, such parameterization constraints may mitigate or delay gradient decay, but they cannot make the gradient modes robust to nonlinear perturbations because the same input-dependent factor $D$ directly perturbs the coefficients governing modal decay. For instance, even orthogonal RNNs, which are initialized with eigenvalues on the unit circle $|\lambda(W_h)| = 1$, are not robustly marginally stable once nonlinearities are introduced: the input-dependent factor $D$ perturbs the transition matrix to $DW_h$, generically shifting its spectrum away from the unit circle (leading to contractive or expansive gradient modes depending on the realization of $D$). Fig. \ref{fig:gradient_plots} demonstrates the lack of robustness in HORNNs for finite $d$, showing the vanishing of gradients over long sequences.

\begin{SCfigure}[\sidecaptionrelwidth][t]
    \centering    \includegraphics[width=0.6\linewidth]{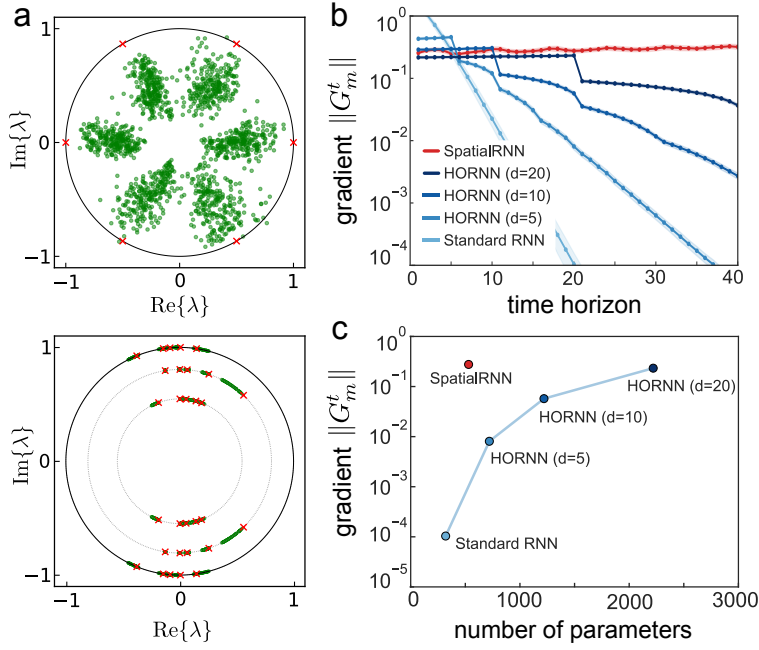}
    \caption{Gradient analysis of SpatialRNNs and HORNNs. 
    (\textbf{a})~Eigenvalue spectrum of the gradient dynamics for standard RNNs (top) and SpatialRNNs (bottom). Red and green markers indicate the eigenvalues computed for $D=\mathbb I$ and $D\in\mathcal D$, respectively. The spectrum of standard RNNs is sensitive to variations in $D$, whereas SpatialRNNs satisfying Theorem~\ref{thm:Stable-Neutral-Mixing} allow for arbitrary eigenvalue placement with invariant magnitude, including on the unit circle.
    (\textbf{b})~Gradient norm versus time horizon.
    (\textbf{c})~Gradient norm (at lag $m=20$) versus the number of parameters (with 8 neurons).
    The results in (b,c) are averaged over 10 trajectories with uniformly sampled inputs.
    }
    \label{fig:gradient_plots}
\end{SCfigure}

\subsection{Robustness of SpatialRNN Gradients}\label{subsec:spatialRNN_grads}

The gradient dynamics in SpatialRNNs is given by (see Appendix~\ref{apx:spatialRNN_gradients} for the derivation): 
\begin{equation}\label{eq:SpatialRNNfullrecurrence}
    G_m^t = D^t_{m} W_\psi \Psi_m^t, \quad m\geq 1,
\end{equation}
where $D_{m}^t = \text{diag}(\Phi'(W_x x_{t+m} + W_\psi \psi_{t+m} + b_h))$, and $\Psi_m^t$ follows the recurrence
\begin{align}\label{eq:SpatialRNNgradients}
    \begin{pmatrix} \Psi_{m}^t \\ \Psi_{m-1}^t\end{pmatrix} = \begin{pmatrix} \overline{W} + B D_{m-1}^t W_\psi & \overline{\overline{W}} \\ \mathbb I & 0 \end{pmatrix} \begin{pmatrix} \Psi_{m-1}^t \\ \Psi_{m-2}^t \end{pmatrix}, \quad m\geq 2,
\end{align}
initialized with $\Psi_{1}^t = B$ and $\Psi_0^t = 0$.
In contrast to RNNs and HORNNs, the SpatialRNNs gradient dynamics are governed by the \textit{second-order} recurrence system \eqref{eq:SpatialRNNgradients}
. In this formulation, the nonlinearity term $D^t_{m}$ acts as a structural perturbation to the gradient dynamics rather than a scaling factor applied at each step. This distinction is crucial as the asymptotic behavior of the gradients is determined by the spectral properties of the transition matrix associated with $\Psi_m^t$, rather than by the instantaneous diagonal entries of $D^t_{m}$. Under a frozen-time approximation, we can derive the sufficient conditions for \textit{robust} gradient behavior of the SpatialRNN. In particular, the following result shows that the eigenvalues $\lambda \in \mathbb C$ of the resulting companion matrix \eqref{eq:SpatialRNNgradients} can be constrained to lie on circles of \textit{fixed radius} $\alpha_i$ which are invariant to the specific realization of $D_{m-1}^t \in \mathcal D$. 

\begin{theorem}[Robustness of SpatialRNN gradients]\label{thm:Stable-Neutral-Mixing}
Let the matrices $\overline{W}$ and $\overline{\overline{W}}$ be defined as block upper-triangular matrices with blocks sizes $n_i$, $i=1,\ldots,k$, such that 
\begin{equation}
    \overline{W} = \begin{pmatrix}
            \overline{W}_1 & \star & \star \\
            0 & \ddots & \star \\
            0 & 0 & \overline{W}_k
        \end{pmatrix} \quad \text{and} \quad
        \overline{\overline{W}} = \begin{pmatrix}
            -\alpha_1\mathbb{I}_{n_1} & \star & \star \\
            0 & \ddots & \star \\
            0 & 0 & -\alpha_{k}\mathbb{I}_{n_k}
        \end{pmatrix},
\end{equation}
where $\overline{W}_i\in\mathbb R^{n_i\times n_i}$, $\alpha_i\in\mathbb R_{>0}$, $\sum_{i=1}^k n_i = n$ is the model size, and $\star$ denotes arbitrary entries.
Suppose the following conditions are satisfied:
    \begin{enumerate}[leftmargin=*]
        \item $\overline{W_i} = \overline{W_i}^\top$, $B = \operatorname{blkdiag}\{B_1,\ldots,B_k\}$, for block matrices $B_i \in \mathbb R^{n_i\times n_i}$, and $W_\psi = \Sigma B^\top = \operatorname{blkdiag}\{W_{\psi,1},\ldots,W_{\psi,k}\}$ for some diagonal matrix $\Sigma \in \operatorname{diag}(\mathbb R^n)$.
        
        \item $\|\overline W_i\|_2 + \overline{\Phi'}\|B_i\|_2\|W_{\psi,i}\|_2 < 2\sqrt{\alpha_i} \ \forall i$, where $\overline{\Phi'} := \max\{|\Phi'_{\min}|, |\Phi'_{\max}|\}$.
    \end{enumerate}
    Then, for any matrix $D_{m-1}^t\in\mathcal D$, the companion matrix of the SpatialRNN gradient dynamics \eqref{eq:SpatialRNNgradients} has exactly $2n_i$ simple eigenvalues of moduli $|\lambda|=\sqrt{\alpha_i}$ for each block $i = 1, \dots, k$.
\end{theorem}
\begin{proof}
    See Appendix \ref{apx:proofs}. The proof exploits the block-triangular structure of $\overline{W}$ and $\overline{\overline{W}}$ to factor the characteristic eigenvalue equation into a product of $k$ independent polynomials. Condition 1 ensures that the eigenvalues of $\overline W_i + B_i D^t_{m-1} W_{\psi,i}$ are real, while condition 2 ensures that the discriminant of the resulting quadratic equations remains negative, leading to underdamped (oscillatory) modes. Consequently, the eigenvalues form complex conjugate pairs whose magnitude is determined by the constant $\sqrt{\alpha_i}$.
\end{proof}


The block-triangular structure of the weight matrices enforces a directed information flow among the $k$ blocks in the PDE dynamics. Within each block, neurons interact fully, while information flows strictly upstream among blocks. This structure allows the parameter $\alpha_i$ to directly control the gradient dynamics at each block $i$. By choosing $\alpha_i=1$, the SpatialRNN achieves marginal stability across all admissible nonlinear perturbations, enabling long-term gradient propagation and memory retention. Setting $\alpha_i < 1$, in turn, leads to exponential decay at the controlled rate, providing an explicit forgetting mechanism at a desired rate.
For a SpatialRNN satisfying the conditions of Theorem~\ref{thm:Stable-Neutral-Mixing} (with $\alpha_1 = 1$ and a single block $k=1$), Fig.~\ref{fig:gradient_plots}b shows that gradient norms remain stable over long sequences of time-varying matrices $D_m^t \in \mathcal{D}$.
Importantly, the parameters $\alpha_i$ can either be fixed a priori or learned via BPTT.

\smallskip
\begin{remark}
The wave equation directly satisfies the conditions in Theorem~\ref{thm:Stable-Neutral-Mixing}  for marginal stability. The system reduces to a single block ($k=1$), with $\alpha_1 = 1$, $\overline W = 2\mathbb I + c^2\Delta t^2 W_\Delta$, and $\overline{\overline W} = -\mathbb I$.
\end{remark}

\smallskip
\begin{remark}
Given the equivalence between SpatialRNNs and $\infty$-order RNNs (Theorem~\ref{thm:SpatialRNN_HORNN}), it follows that, unlike their finite-order counterparts, $\infty$-order RNNs can achieve robust marginal stability when the delay matrices $\{W_k\}_{k\geq 1}$ are appropriately constrained (see Appendix~\ref{apx:spatialrnn-hornn}).
\end{remark}

SpatialRNNs, similarly to SSMs, can improve the training memory efficiency of recurrent models. In standard RNNs, BPTT demands $\mathcal{O}(T)$ memory to cache intermediate hidden states for gradient computation. For SpatialRNNs, the medium dynamics are invertible, so that intermediate states can be reconstructed backward in time from the final states, reducing this requirement to $\mathcal{O}(1)$ in exact arithmetic. In finite precision, however, the numerical reliability of this reconstruction depends on the dynamical regime. In the non-dissipative case ($\alpha_i=1,\,\forall i$), the medium dynamics preserve a symplectic form~\citep{marsden2001discrete}, enabling backward reconstruction to remain accurate to machine precision. By contrast, dissipative blocks ($\alpha_i<1$) exponentially amplify round-off errors under backward reconstruction and therefore require caching or checkpointing. Further details on this analysis are presented in Appendix~\ref{apx:symplectic}.

\section{Experiments}\label{sec:experiments}
We evaluate the memory capabilities of the SpatialRNN on benchmark datasets with long temporal dependencies, comparing it against other recurrent models. Experimental details (datasets, model architectures, and training procedures) are included in Appendix~\ref{apx:exp_det}; no regularizers (e.g., dropout or normalization) have been used, highlighting the intrinsic capabilities of the SpatialRNN without auxiliary optimization. All experiments use standard BPTT.
Code to reproduce all experiments can be found in \url{anonymous.4open.science/r/SpatialRNN-code-748C}.
%
%

\textbf{Copy task.}
As a proof-of-concept, we compare SpatialRNNs and HORNNs on the copy task~\citep{hochreiter1997long,arjovsky2016unitary},
%
%
a benchmark designed to test the ability to store information over long sequences and accurately retrieve it. The input sequence $X = [X_{\text{data}}, X_{\text{dist}}, X_{\text{trig}}] \in \{0, \pm 1\}^{D \times T}$ consists of $L=5$ binary tokens $X_{\text{data}}$ with dimension $D=4$ (representing the data to be stored), followed by $T-2L$ distractors $X_{\text{dist}}$ (padded with zeros), and finally a sequence of $L$ triggers $X_{\text{trig}}$ (padded with ones). The models are trained to reconstruct the $X_{\text{data}}$ at the final $L$ steps. Here, we evaluate the model performance for sequences of increasing length $T \in \{2^p\}_{p=1}^{10}$.


Fig.~\ref{fig:copy} compares the SpatialRNN against HORNNs (64 neurons) with increasing receptive fields, from $d=1$ (standard RNN, 5.5k parameters) to $d=15$ (17.7k parameters). For the SpatialRNN, we impose the conditions of Theorem~\ref{thm:Stable-Neutral-Mixing} with $k=4$ block matrices of sizes $n_i\in\{32, 8, 8, 8\}$ (6.8k parameters) and $\alpha_1=1$  to guarantee marginal stability. 
Consistent with the numerical
\begin{figure}[t]
    \centering
    \includegraphics[width=0.6\linewidth]{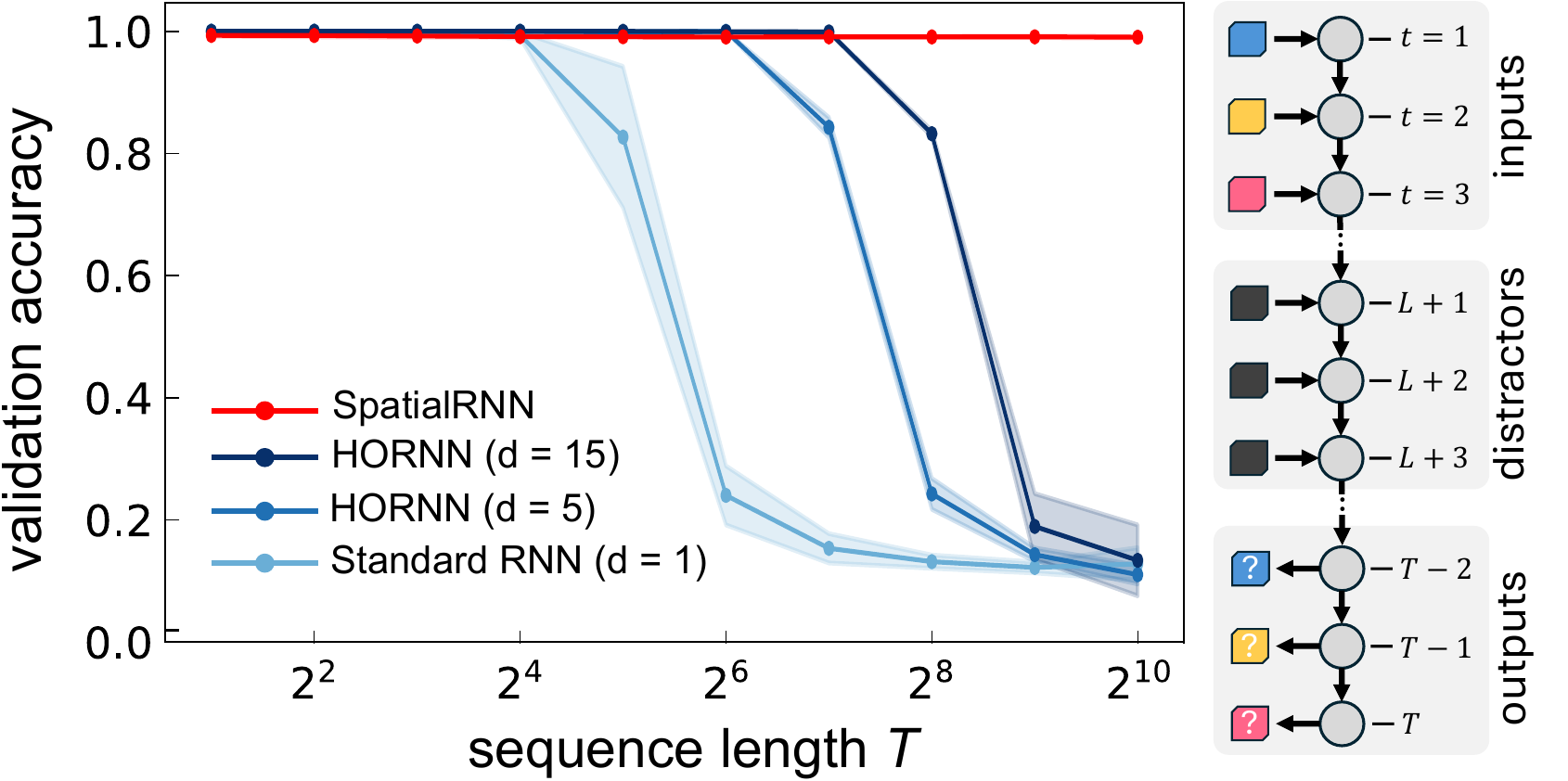}
    \caption{Performance of SpatialRNN and HORNNs on the copy task problem, shown as a function of sequence length $T$.}
    \label{fig:copy}
\end{figure}
gradient analysis in Fig.~\ref{fig:gradient_plots}b, 
HORNN performance degrades with increasing $T$,  with larger receptive fields $d$ delaying, but not preventing, this deterioration. In contrast, the SpatialRNN maintains near-perfect (${\approx}99\%$) validation accuracy across all sequence lengths, despite having fewer parameters. 
This behavior is consistent with the robust marginal stability established in Theorem~\ref{thm:Stable-Neutral-Mixing} (with $\alpha_1=1$), which enables long-term memory retention. The HORNN comparison also serves as an ablation of state dimension versus PDE-mediated structure: A $d$-order HORNN has a lifted state of dimension $nd$, so a second-order HORNN already matches the order and state dimension of the SpatialRNN ($2n$, corresponding to $\psi_t,\psi_{t-1}$), while models with $d\geq5$ greatly exceed it. Nevertheless, all finite-order HORNNs remain subject to Proposition~\ref{prop:hornn_instability} and eventually lose accuracy in the copy task for large $T$, indicating that the observed advantage of the SpatialRNN cannot be attributed simply to higher recurrence order or increased state dimension. Interestingly, even when $\alpha_i$ is left unconstrained for all blocks, training exhibits a sharp transition in performance when the model learns a mix of blocks with $\alpha_i < 1$ and $\alpha_j \approx 1$, allowing it to simultaneously retain information and forget distractors (see Appendix~\ref{app.emergencealpha} and  Fig.~\ref{fig:alpha_evolution}).

\smallskip
\textbf{Permuted sequential MNIST (psMNIST).}
We evaluate long-range dependency modeling on the psMNIST classification task~\citep{le2015simple,arjovsky2016unitary}. Here, MNIST images are first flattened into sequences of length $T=784$ with feature dimension $D=1$, such that pixels are fed to the model one at a time. The sequence is then normalized and randomly permuted, destroying the local spatial structure of the images and forcing the model to rely entirely on long-term temporal memory for classification.

\begin{table}[h]
\centering
\begin{tabular}{lcc}
\toprule
\textbf{Model (\# units)} & \textbf{Acc. (\%)} & \textbf{Params} \\
\midrule
unitary RNN (512)~\citep{arjovsky2016unitary}& 91.4 (max) & 9k \\
LSTM (256)~\citep{helfrich2018orthogonal} & 92.9 (max) & 270k \\
GRU (256)~\citep{chang2017dilated} & 94.1 (max) & 200k \\
NWM (256)~\citep{keller2023neural} & $94.8 \pm 1.1$ & 50k \\
coRNN (256)~\citep{keller2023neural} & $95.0 \pm 2.4$ & 134k \\
\rowcolor{gray!20}SpatialRNN (88) & $95.5 \pm 0.2$ & 13k \\
\rowcolor{gray!20}SpatialRNN (152) & $96.3 \pm 0.1$ & 36k \\
\bottomrule
\end{tabular}
\caption{psMNIST performance. Mean ± std over 3 runs when available; otherwise, maximum accuracy. Only single-layer architectures are accounted for.}
\label{tab:psmnist_results}
\end{table}

We evaluate two SpatialRNN configurations: an 88-neuron model (block sizes $\{64, 8, 8, 8\}$, 12.8k parameters) and a 152-neuron model (block sizes $\{128, 8, 8, 8\}$, 36.5k parameters). In both cases, we enforce $\alpha_1=1$ on the larger block to guarantee marginal stability. The SpatialRNNs achieve the highest accuracy among the single-layer recurrent baselines, $96.3 \pm 0.1$\% (Table~\ref{tab:psmnist_results}), outperforming standard gated architectures such as LSTM 
and GRU, 
as well as recent oscillatory models including NWM 
\& coRNN. 
Notably, this performance is achieved with roughly four to ten times fewer parameters.

\smallskip
\textbf{Noise padded CIFAR-10 (npCIFAR10).}
In this benchmark \citep{chang2019antisymmetricrnn}, CIFAR-10 images (of size $32\times 32\times 3$) are normalized and reshaped into sequences of length $L=32$ with feature dimension $D=96$, so that each timestep contains one image row across all color channels. Standard normal noise is then appended to reach a total sequence length of $T=1000$, acting as distractors.

We evaluate two SpatialRNN configurations: a 88-neuron model (block sizes $\{64, 8, 8, 8\}$, 21.2k parameters) and an 152-neuron model (block sizes $\{128, 8, 8, 8\}$, 50.9k parameters), where we once again enforce $\alpha_1=1$ for marginal stability. The 88-neuron SpatialRNN achieves competitive accuracy while operating with a fraction of the parameters required by standard gated and continuous-time baselines (Table~\ref{tab:npcifar10_results}). Most importantly, Fig.~\ref{fig:cifar} shows that performance is invariant to the amount of appended noise, indicating that accuracy is mostly dictated by the spatial feature extraction of the initial sequence rather than by temporal degradation over the sequence length $T$.



\begin{table}[h]
  \centering
        \begin{tabular}{lcc}
        \toprule
        \textbf{Model (\# units)} & \textbf{Acc. (\%)} & \textbf{Params} \\
        \midrule
        LSTM (128)~\citep{kag2020rnns} & 11.6 & 116k \\
        GRU (128)~\citep{keller2023traveling} & 43.8 & 88k \\
        iRNN (529)~\citep{keller2023traveling} & 51.3 & 336k \\
        wRNN (256)~\citep{keller2023traveling} & 55.0 & 435k \\        \rowcolor{gray!20}SpatialRNN (88) & 57.2 & 21k \\
        LRNN (128)~\citep{keller2023traveling} & 57.4 & 46k \\
        \rowcolor{gray!20}SpatialRNN (152) & 57.8 & 51k \\
        coRNN (128)~\citep{rusch2020coupled} & 59.0 & 46k \\
        LEM (128)~\citep{keller2023traveling} & 60.5 & 116k \\
        \bottomrule
        \end{tabular}
        \caption{npCIFAR10 performance for $T=1000$. Maximum test accuracy over 3 trials.}
\label{tab:npcifar10_results}
\end{table}

\begin{figure}[h]
  \centering
  \includegraphics[width=0.4\linewidth]{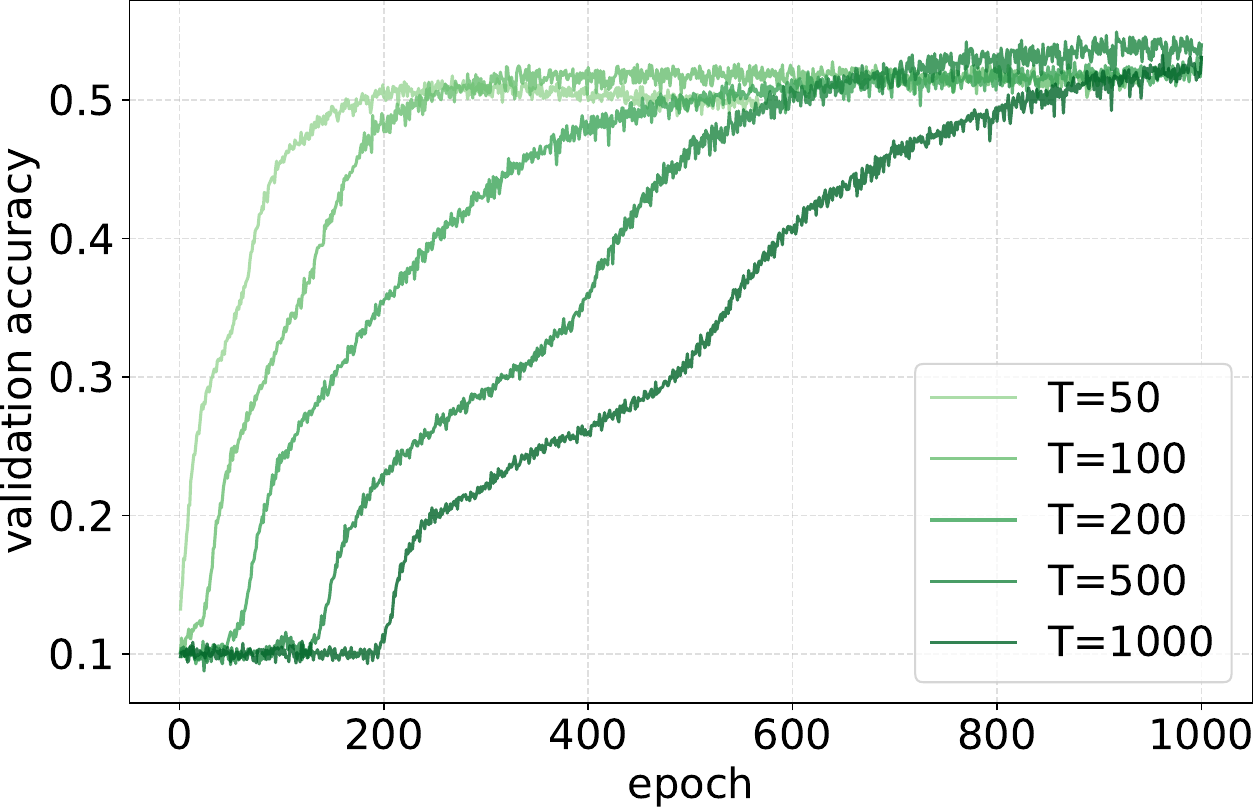}
  \caption{SpatialRNN validation accuracy on the npCIFAR10 task for varying $T$.}
  \label{fig:cifar}
\end{figure}

\section{Conclusion}\label{sec:conclusion}

In this work, we introduced a spatial computing model in which neuron states communicate through a dynamically evolving spatial medium rather than through direct recurrent connections. We showed that the medium state implicitly encodes the entire history of past hidden states, giving the model an unbounded receptive field with a fixed number of parameters. The second-order dynamics induced by the medium also provide a structural mechanism for controlling gradient propagation: under suitable constraints, the gradient modes can be confined to circles of prescribed radii despite input-dependent nonlinear perturbations, enabling long-term memory retention. Together, these results establish spatial dynamics as a parameter-efficient mechanism for combining long-range temporal dependencies with robust gradient propagation.

From a theoretical perspective, our analysis identifies a structural principle underlying long-term memory in recurrent architectures. We show that robustness to input-dependent nonlinear perturbations is governed not simply by the spectrum or parameterization of the recurrence matrices, but by how the nonlinearity enters the gradient dynamics (Proposition~\ref{prop:hornn_instability} and Theorem~\ref{thm:Stable-Neutral-Mixing}). In finite-order RNNs, the nonlinear term directly perturbs the coefficients governing modal decay, precluding robust marginal stability. In the SpatialRNN, the second-order recurrence separates the input-dependent perturbation from the coefficients that determine the magnitudes of the gradient modes, allowing those magnitudes to remain invariant. This distinction provides a general design principle for recurrent architectures capable of sustaining long-lived gradient modes.

\textbf{Related work.}
Neural ordinary differential equations \citep{chen2018neural} generalize RNNs by extending discrete-time recurrences to continuous-time dynamics, enabling richer state evolution and improved control over gradient behavior. A closely related line of work considers oscillatory \citep{rusch2020coupled, rusch2022graph} and second-order \citep{effenberger2025functional,rusch2021unicornn} dynamical systems, in which neurons evolve as coupled oscillators. When constrained to local interactions, these models give rise to wave-like dynamics, as in neural wave machines~\citep{keller2023neural}, which have demonstrated strong performance on spatially structured data~\citep{jacobs2025traveling} and long-sequence tasks~\citep{keller2023neural}. These models often exploit traveling waves to encode and propagate temporal information \citep{benigno2023waves,jacobs2025traveling,shervani2026emergent}. However, their stability depends directly on discretization parameters, and thus the gradient dynamics exhibit residual exponential scaling with the time step $\Delta t$~\citep{rusch2022graph}. Related ideas have also been explored in other state-based architectures, including reservoir computing~\citep{benigno2023waves} and SSMs~\citep{rusch2024oscillatory}. Our approach is closely related to these frameworks but differs fundamentally since traveling waves are encoded in a shared medium, rather than emerging from interactions between local oscillatory units; these approaches coincide only under certain restrictive conditions (Appendix~\ref{app.coupledoscillator}).

\textbf{Limitation.}
A key limitation of our analysis is the gap between the frozen-time and the full LTV gradient dynamics. The frozen-time guarantees of Theorem~\ref{thm:Stable-Neutral-Mixing} are necessary, but not sufficient, for marginal stability under arbitrary switching, since products of matrices with unit-modulus spectra may still grow or decay. 
%
%
Gradients may thus grow over long horizons, although we observe empirically that this drift is substantially attenuated and does not appear exponential. In particular, the exact Jacobians of a trained SpatialRNN change by only around $6\%$ over $783$ steps along psMNIST trajectories (see Appendix~\ref{apx:ltv_empirical} for further numerical analysis). Asymptotic stability of the SpatialRNN gradients can be enforced for the full LTV system under additional structural constraints (Theorem~\ref{thm:guas}), yet establishing necessary and sufficient conditions for LTV marginal stability remains an open problem.

More broadly, our framework provides an explicit mechanism through which spatiotemporal dynamics endow recurrent models with long-lived memory\textemdash a longstanding goal in the AI community \citep{keller2026spatiotemporal}. However, memory alone does not determine overall performance. The information capacity of the medium is effectively bounded by its dimension and numerical precision (Remark~\ref{rmk:receptive_field}), while feature extraction and expressivity remain architecture-dependent. This result is evident in the npCIFAR10 benchmark (Fig. \ref{fig:cifar}), where the SpatialRNN robustly maintains performance across varying sequence lengths yet remains limited in classification accuracy. Computational efficiency presents another trade-off: Each SpatialRNN step involves a few more matrix-vector products than standard RNNs, and its nonlinear recurrence is evaluated sequentially (unlike the associative scans in linear SSMs). 

\textbf{Future work.} A natural direction for future work is thus to integrate spatial computing mechanisms into more expressive state-based architectures such as SSMs, which can scale efficiently on modern hardware \citep{gonzalez2025predictability}; for instance, by lifting their recurrence equations to second order. We expect this paradigm to be best suited to streaming problems requiring a fixed-size state, long retention, and selective forgetting. Relevant applications include sensor monitoring, long-horizon control, and spatiotemporal signal prediction, where the locality and geometry of the medium can provide useful inductive biases. We also expect this formulation to find applications in physical computing, where wave-based systems, such as those exhibiting solitons, serve as computational substrates \citep{hughes2019wave,marcucci2020theory,montanari2026energy}. 
In this context, the symplectic structure induced by non-dissipative PDE dynamics can provide an additional advantage \citep{chen2019symplectic}, enabling exact backward state reconstruction and reducing the memory cost of backpropagation to $O(1)$ with respect to the sequence length $T$.

\section*{Broader Impact Statement}
This work studies how a spatially evolving medium, governed by discretized PDEs, can mediate memory retention and gradient propagation in recurrent models. Its contributions are theoretical results and small-scale benchmarks, and we do not foresee direct negative societal impacts beyond those common to general-purpose sequence models. By retaining long-range dependencies in a fixed-size state, rather than through gating mechanisms or a key-value memory that grows with sequence length, the SpatialRNN may help reduce the memory cost of long-horizon inference. We also expect the formulation to be relevant to analog and wave-based physical computing substrates, where the medium dynamics could be realized directly in hardware. Whether these advantages translate into energy savings in practice remains to be established.

\subsubsection*{Acknowledgments}
This work was supported by the ARO Grant No.~W911NF-24-1-0228.


\bibliography{biblio_2026}
\bibliographystyle{tmlr}


\appendix

\section{Supplementary Theoretical Analysis}
\label{app.theory}

\subsection{Proofs of Main Theoretical Results}\label{apx:proofs}

We provide the proofs for the main theoretical statements presented in the main text.

\begin{proof}[Proof of Theorem \ref{thm:SpatialRNN_HORNN}]
    Consider the medium state $\psi$ given by Eq.~\eqref{eq:spatial_convolution} and substitute it into both sides of the LDE~\eqref{eq:lde}:
    \begin{align*}
     \sum_{k=1}^{t-1} W_k h_{t-k} = \overline W \sum_{k=1}^{t-2} W_k h_{t-1-k} + \overline{\overline W} \sum_{k=1}^{t-3} W_k h_{t-2-k} + B h_{t-1} \\
    = \overline W \sum_{k=2}^{t-1} W_{k-1} h_{t-k} + \overline{\overline W} \sum_{k=3}^{t-1} W_{k-2} h_{t-k} + B h_{t-1}.
            \end{align*}
    Equating the coefficients in both sides for each $h_{t-k}$, we obtain $W_1 = B$ for $k=1$, $W_2 = \overline{W}W_1$ for $k=2$, and $W_k = \overline W W_{k-1} + \overline{\overline W} W_{k-2}$ for $k\geq 3$.
\end{proof}

\begin{proof}[Proof of Proposition \ref{prop:HORNN_det_0}]
    Consider the frozen-time recurrence system $G_m^t = D \sum_{k} W_k \, G_{m-k}^t$, for $m\geq 1$, and $G_0^t = \mathbb I$. Define the $z$-transform $\hat G(z) := \sum_{m=0}^\infty G^t_m z^{-m}$. Then, it follows from the recurrence system that
    \begin{align}
        \hat G(z) &= G_0 + \sum_{m=1}^\infty D\left( \sum_{k=1}^{m}  W_k G^t_{m-k} \right) z^{-m} \\
        &= \mathbb I + D \sum_{k=1}^\infty W_k z^{-k} \sum_{m=k}^\infty G^t_{m-k} z^{m-k}\\
        &= \mathbb I + D \left(\sum_{k=1}^\infty W_k z^{-k}\right)\left( \sum_{j=0}^\infty G^t_{j} z^{-j}\right)\\
        &= \mathbb I + D W(z) G(z).
    \end{align}
    Re-arranging to make the function explicit, we get 
    \begin{equation} \label{eq:Gz}
        G(z) = (\mathbb I - D \hat W(z))^{-1}.
    \end{equation} 
    It is well-known
    %
    %
    that the poles of $G(z)$ are precisely the values $z\in\mathbb C$ that satisfy the characteristic equation $\det(\mathbb I - D \hat W(z))=0$~\citep{ogata1995discrete}. These values characterize the modes of the linear recurrence system, and therefore correspond to the eigenvalues $\lambda$ governing the asymptotic behavior of the gradient dynamics.
\end{proof}

\begin{proof}[Proof of Proposition \ref{prop:hornn_instability}]
We proceed by contradiction. Suppose the system is robustly marginally stable, i.e., all roots $\lambda$ of the characteristic equation \eqref{eq:HORNN_det_0} satisfy $|\lambda| = 1$ for every $D \in \mathcal{D}$.

Consider the subset $\Omega = \{D \in \mathcal{D} \mid D = \gamma \mathbb{I}, \, \gamma \, \in \, \mathbb R_{\neq 0} \}$ and the auxiliary domain $\overline \Omega = \{D^{-1} \mid D \in \Omega\}$. Since $\Phi'$ takes values arbitrarily close to 0, 
it follows that $\overline \Omega$ contains scalar matrices of the form $\Delta = \delta \mathbb{I}$ for arbitrarily large (in norm) $|\delta| > 0$.
Note that all matrices $D \in \Omega$ are invertible. Taking $\Delta = D^{-1}$ we can re-write the characteristic equation \eqref{eq:HORNN_det_0} for such matrices
\begin{equation}
    \det(I - D\hat{W}(\lambda)) = \det(D(\Delta - \hat{W}(\lambda))) = \det(D)\det(\Delta - \hat{W}(\lambda)) = 0.
\end{equation}
Since $D$ is nonsingular, this equation is equivalent to $\det(\Delta - \hat{W}(\lambda)) = 0$.

Recall that $\hat W(\lambda) = \sum_{k=1}^d  W_k \lambda^{-k}$. Now, consider the unit circle $\mathbb{S}^1 = \{z \in \mathbb{C} : |z| = 1\}$, and let us define $\rho_{\hat{W}}(\lambda) = \max_{i} |\mu_i(\hat{W}(\lambda))|$ as the spectral radius of the polynomial matrix $\hat{W}(\lambda)$, where $\mu_i$ denotes its eigenvalues. 
Since $\hat{W}(\lambda)$ is a finite polynomial in $\lambda^{-1}$, it is continuous on $\mathbb{S}^1$. Because $\mathbb{S}^1$ is compact, $\rho_{\hat{W}}(\lambda)$ achieves a maximum value $M := \sup_{\lambda \in \mathbb{S}^1} \rho_{\hat{W}}(\lambda) < \infty$. We can thus choose a matrix $\Delta^* = \delta \mathbb{I} \in \Omega$ such that $\delta > M$, and take $D^*=(\Delta^*)^{-1}\in\mathcal D$. The characteristic equation becomes $\det(\delta \mathbb{I} - \hat{W}(\lambda)) = 0$, implying that $\lambda$ is a root if and only if $\delta$ is an eigenvalue of $\hat{W}(\lambda)$.

Under the assumption of robust marginal stability, there must exist a root $\lambda_0$ on the unit circle (i.e., $|\lambda_0|=1$). Then, $\delta$ is an eigenvalue of $\hat W(\lambda_0)$, which implies that
\begin{equation}
    \delta \le \rho_{\hat{W}}(\lambda_0) \le M.
\end{equation}
This contradicts the choice $\delta > M$. Therefore, there exists a matrix $D = (\Delta^*)^{-1}\in\mathcal D$ such that the characteristic equation has no roots on the unit circle. Consequently, the system is not robustly marginally stable.
\end{proof}

\begin{proof}[Proof of Theorem \ref{thm:Stable-Neutral-Mixing}]
    The second-order system~\eqref{eq:SpatialRNNgradients} has characteristic equation
    \begin{equation}
        \det(z^2\mathbb{I} - z(\overline{W} + B D_{m-1}^t W_{\psi}) - \overline{\overline{W}}) = 0.
    \end{equation}
    By assumption, matrices ($\overline{W}, \overline{\overline{W}}, B,W_\psi$) have block upper-triangular structure with the same block partition $(n_1,\ldots,n_k)$. It thus follows that the matrix $M(z) := z^2\mathbb I - z(\overline W + B D_{m-1}^t W_\psi) - \overline{\overline W}$ is also block upper-triangular. Therefore, 
    \begin{equation}\label{eq:polesGzEqprod}
        \det(M(z)) = \prod_{i=1}^k \det(M_i(z)),
    \end{equation}
    where $M_i(z)$ denotes the $i$th diagonal block of $M(z)$.
    Using the assumption $\overline{\overline W}_i = -\alpha_i \mathbb I_{n_i}$, Eq.~\eqref{eq:polesGzEqprod} implies that $\det(M(z)) = 0$ if and only if
    \begin{equation}\label{eq:blk_diag_eigval_equation}
        \det(z^2\mathbb I - z(\overline W_i + B_i D_{m-1}^t W_{\psi, i}) + \alpha_i \mathbb I_{n_i}) = 0, \,\, \text{for some } i\in\{1,\ldots,k\}.
    \end{equation}
    Therefore, the set of eigenvalues of the full system $M(z)$ is equal to the union of the eigenvalues of each subsystem $M_i(z)$.
    
    For each block $i$, let $N_{i} = \overline{W}_i + B_i D_{m-1}^t W_{\psi, i}$. Since $\overline{W}_i$ is symmetric, $D_{m-1}^t$ is diagonal, and $W_{\psi, i} = \Sigma_i B_i^\top$ (with $\Sigma_i$ the $i$th diagonal block of $\Sigma$), $N_{i}$ is a real symmetric matrix. It thus holds that $N_i$ is diagonalizable and has real eigenvalues $\mu_j$, $j=1,\dots,n_i$. We can thus decouple the characteristic equation \eqref{eq:blk_diag_eigval_equation} into $n_i$ independent quadratic equations:
    \begin{equation}\label{eq:zeqsproof}
        z^2 - \mu_j z + \alpha_i = 0, \quad \text{for } j=1,\dots,n_i.
    \end{equation}
    The roots of these quadratic equations are given by
    \begin{equation}
        z_{j, \pm} = \frac{\mu_j \pm \sqrt{\mu_j^2 - 4\alpha_i}}{2}.
    \end{equation}
    Given that $D_{m-1}^t \in \mathcal D$, we have $\|D_{m-1}^t\|_2\leq \overline{\Phi'}$. 
    Therefore, 
    \begin{align}
        \|\overline W_i\|_2 + \overline{\Phi'}\|B_i\|_2\|W_{\psi,i}\|_2 < 2\sqrt{\alpha_i} \\
         \|\overline W_i + B_iD_{m-1}^tW_{\psi,i}\|_2 < 2\sqrt{\alpha_i} \\
          |\mu_j| < 2\sqrt{\alpha_i}.
    \end{align}
    
    This condition ensures that the discriminant of Eq.~\eqref{eq:zeqsproof} is negative (i.e., $\mu_j^2 - 4 \alpha_i < 0$), which ensures that the roots $z_{j, \pm}$ are simple and have a nonzero imaginary part. Moreover, the absolute values of the roots can be seen to lie on the complex circle of radius $\sqrt{\alpha_i}$,
    \begin{equation}
        |z_{j, \pm}|^2 = \left(\frac{\mu_j}{2}\right)^2 + \left(\frac{\sqrt{4\alpha_i -\mu_j^2}}{2}\right)^2  = {\alpha_i},
    \end{equation}
    for any matrix $D_{m-1}^t\in\mathcal D$.
\end{proof}


\subsection{Derivation of the SpatialRNN Gradient Dynamics}\label{apx:spatialRNN_gradients}

We fix the reference time $t$ and define the Jacobian matrices $G_{m}^t = \frac{\partial h_{t+m}}{\partial h_t}$ and $\Psi^t_m = \frac{\partial \psi_{t+m}}{\partial h_t}$. Using the chain rule, it follows that the gradients of the SpatialRNN \eqref{eq:SpatialRNN} are given by the coupled LTV recurrence equations:
\begin{equation}
    \begin{cases}
    G_{m}^t = D_{m}^t W_\psi \Psi^t_m,\\
   \Psi^t_m = \overline{W} \Psi^t_{m-1} + \overline{\overline{W}} \Psi^t_{m-2} + B G_{m-1},  \quad m\geq 2,
    \end{cases}
\end{equation}
where $\Psi^t_{1} = B$, $\Psi^t_0 = 0$, and $D_{m}^t = \text{diag}\left(\Phi'\left(W_{x}x_{t+m} + W_{\psi}\psi_{t+m} + b_{h}\right)\right)$.
Substituting the first equation into the second one, we obtain
\begin{align*}
\Psi^t_m &= \overline{W}\Psi^t_{m-1} + \overline{\overline{W}}\Psi^t_{m-2} + B (D_{m-1}^t W_\psi \Psi^t_{m-1}) \\
&= (\overline{W} + B D_{m-1}^t W_\psi) \Psi^t_{m-1} + \overline{\overline{W}}\Psi^t_{m-2},
\end{align*}
which can be written in matrix form as Eq.~\eqref{eq:SpatialRNNgradients}.

\subsection{Robustness of the Gradient Dynamics for $\infty$-order RNNs} \label{apx:spatialrnn-hornn}

The following result obtains the gradient dynamics for an $\infty$-order RNN whose recurrence matrices are determined by Eq.~\eqref{eq:recurrence}. Following Theorem \ref{thm:SpatialRNN_HORNN}, such $\infty$-order RNN is equivalent to a SpatialRNN, and hence we can directly prove that its gradient dynamics are equivalent to those of the SpatialRNN.

\smallskip
\begin{proposition}\label{prop:Gz_margStab_det_0}
Let $\{G_m^t\}_{m\geq0}$ be the gradient sequence defined by the HORNN recurrence \eqref{eq:HORNNgrad}, with delay transition matrices $W_k = W_\psi\tilde W_k$ such that $\{ \tilde W_k\}_{k=1}^\infty$ satisfy the relation \eqref{eq:recurrence} and $W_\psi$ is invertible. For any frozen-time matrix $D\in\mathcal D$, $\lambda\in\mathbb C$ is an eigenvalue of the recurrence system \eqref{eq:HORNNgrad} if and only if
    \begin{equation}\label{eq:polesGzEq}
        \det(\lambda^2\mathbb I - \lambda(\overline W + B D W_\psi) - \overline{\overline W}) = 0.
\end{equation}
\end{proposition}

\begin{proof} 
Recall that $\hat W(\lambda) = \sum_{k=1}^\infty  W_k \lambda^{-k} = \sum_{k=1}^\infty  W_{\psi} \tilde W_k \lambda^{-k} $ and let us define $z=\lambda^{-1}$. Since $W_{\psi}$ is invertible,  we can factor it out of $\hat W(z)=\sum_{k=1}^{\infty} W_{\psi} \tilde W_k z^k$ as follows:
\begin{align}
W_\psi^{-1} \hat W(z) &= \tilde W_1 z + \tilde W_2 z^2 + \sum_{k=3}^{\infty} \tilde W_k z^k \\
W_\psi^{-1} \hat W(z) &= \tilde W_1 z + \tilde W_2 z^2 + \sum_{k=3}^{\infty} (\overline{W}\tilde W_{k-1} + \overline{\overline{W}}\tilde W_{k-2}) z^k \\
W_\psi^{-1} \hat W(z) &= \tilde W_1 z + \tilde W_2 z^2 + \overline{W}z \sum_{k=3}^{\infty} \tilde W_{k-1} z^{k-1} + \overline{\overline{W}}z^2 \sum_{k=3}^{\infty} \tilde W_{k-2} z^{k-2} \\
W_\psi^{-1} \hat W(z) &= \tilde W_1 z + \tilde W_2 z^2 + z\overline{W}\hat W(z) - z^2\overline{W}\tilde W_1 + z^2\overline{\overline{W}}\hat W(z) \\
W_\psi^{-1} \hat W(z) \left( \mathbb{I} - z\overline{W} - z^2\overline{\overline{W}} \right) &= \tilde W_1 z + (\tilde W_2 - \overline{W}\tilde W_1) z^2 \\
W_\psi^{-1} \hat W(z) \left( \mathbb{I} - z\overline{W} - z^2\overline{\overline{W}} \right) &= B z + \cancel{(\overline{W}B - \overline{W}B)} z^2 \\
\hat W(z) &= W_\psi (\mathbb{I} - z\overline{W} - z^2\overline{\overline{W}})^{-1} B z.
\label{eq:Wz_PDE}
\end{align}

From Eq.~\eqref{eq:HORNN_det_0}, we have that the poles of $G(z)$ are the solutions to $\det(\mathbb I - D\hat W(z)) = 0$. Substituting Eq.~\eqref{eq:Wz_PDE} into it, we have that
\begin{align*}
    \det(\mathbb{I} - D W_{\psi} (\mathbb{I} - z\overline{W} - z^2\overline{\overline{W}})^{-1} B z) &= 0 \\
     \det((\mathbb{I} - z\overline{W} - z^2\overline{\overline{W}}) - B D W_{\psi} z) &= 0 \\
     \det(\mathbb{I} - z(\overline{W} + B D W_{\psi}) - z^2\overline{\overline{W}}) &= 0,
\end{align*}
for every $D\in\mathcal D$.
\end{proof}

\subsection{Emergence of Marginal Stability During SpatialRNN Training}
\label{app.emergencealpha}

Theorem~\ref{thm:Stable-Neutral-Mixing} shows that setting $\alpha_i = 1$ is a required condition to guarantee marginal stability within the $i$-th block. Importantly, $\alpha_i$ can either be fixed a priori or treated as a trainable parameter. In the latter case, we show empirically that the SpatialRNN model can spontaneously learn to set $\alpha\approx 1$ during BPTT training. Fig.~\ref{fig:alpha_evolution} illustrates the evolution of the decay parameters $\alpha_i$ alongside the model's accuracy over the training period. Early in the training, the model fails to capture long-range dependencies, remaining stuck at chance-level performance. The learning dynamics only exhibit a sharp transition when gradient descent pushes at least one parameter (in this case, $\alpha_4$) close to $1$.

\begin{figure}[h]
    \centering
    \includegraphics[width=0.5\linewidth]{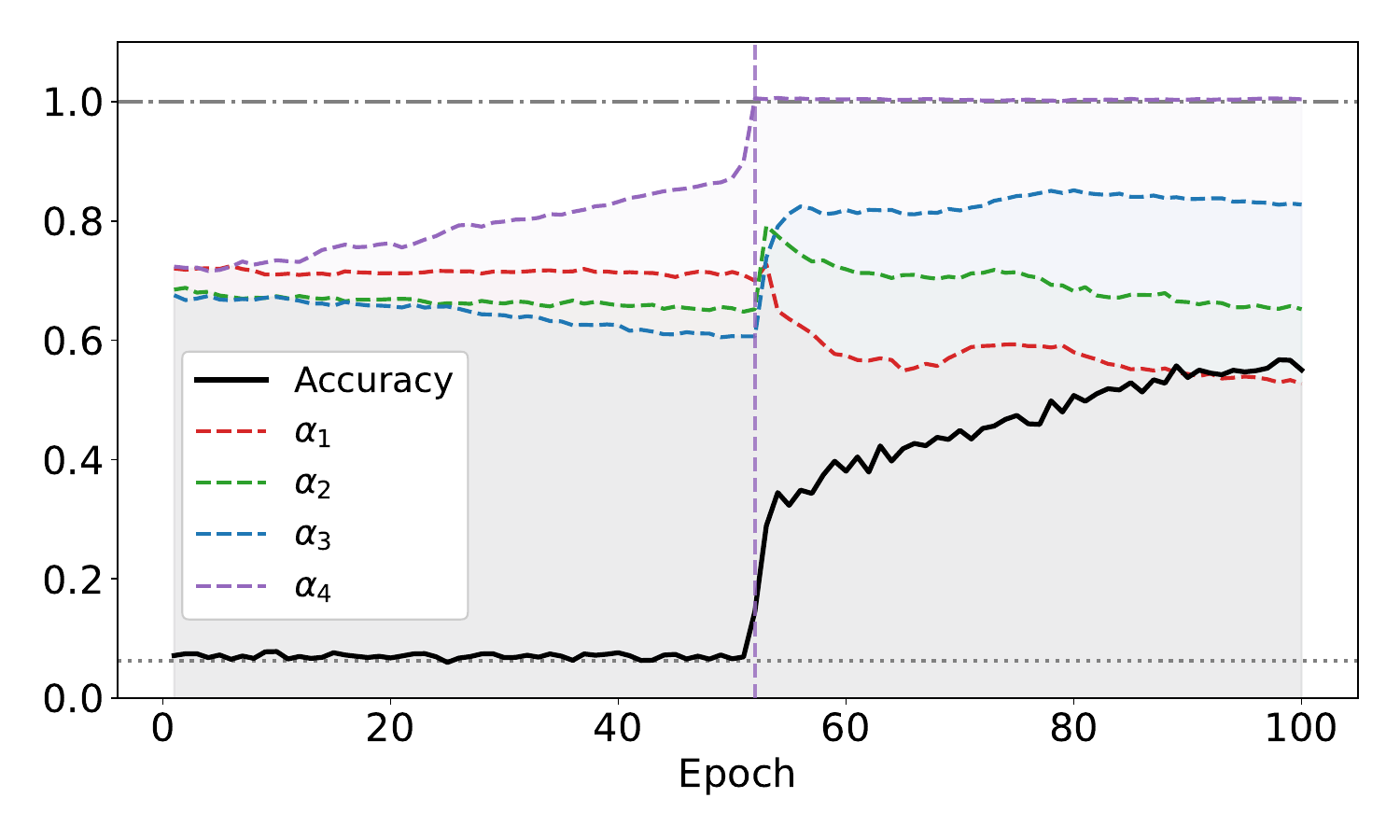}
    \caption{Evolution of the parameters $\{\alpha_i\}_{i=1}^k$ during training on the copy task with sequence length $T=128$. The dashed lines show the learned values of $\alpha_i$ for a SpatialRNN architecture with block sizes [8, 8, 8, 8], while the solid line indicates training accuracy. Performance remains at chance-level until $\alpha_4 \to 1$, after which learning rapidly progresses.}
    \label{fig:alpha_evolution}
\end{figure}

While the SpatialRNN is expressive enough to learn the marginally stable regime ($\alpha \to 1$), training can be slow when starting from a random initialization. To bypass this issue, we explicitly fix $\alpha_i = 1$ for at least one block in all experimental evaluations.

\subsection{Connection between SpatialRNN and Oscillator Networks}
\label{app.coupledoscillator}
%
%

Enforcing a block-triangular structure in the weight matrices allows us to decompose the SpatialRNN's gradient eigenvalue problem~\eqref{eq:SpatialRNNgradients} into $k$ independent blocks, offering robust control over the moduli of the eigenvalues of each block (Theorem~\ref{thm:Stable-Neutral-Mixing}). This construction allows the modes (eigenvectors) to vary at each timestep, promoting information mixing and higher model expressivity.

An alternative method to ensure robust control of the gradient spectrum is to ensure simultaneous diagonalizability of the medium update matrices $\overline W$ and $\overline{\overline W}$. As we show next, \textit{this condition essentially decouples the wave equation into $n$ independent harmonic oscillators}. In this case, the eigenvectors remain fixed over time, preventing information mixing among modes and thereby limiting expressivity. Thus, we do not employ this approach in our experiments.

\begin{theorem}[Robust SpatialRNN gradients through simultaneous diagonalization]
\label{thm:simultaneous_diag}
Let the medium update matrices $\overline{W}$ and $\overline{\overline{W}}$ be simultaneously diagonalizable such that $\overline{W} = S \overline{\Lambda} S^{-1}$ and $\overline{\overline{W}} = S \overline{\overline{\Lambda}} S^{-1}$ for an invertible matrix $S \in \mathbb{R}^{n \times n}$, with $\overline{\Lambda} = \operatorname{diag}(\mu_1, \dots, \mu_n)$ for $\mu_i\in\mathbb R$ and $\overline{\overline{\Lambda}} = \operatorname{diag}(-\alpha_1, \dots, -\alpha_n)$ for $\alpha_i \in \mathbb{R}_{>0}$. Suppose the following conditions are satisfied:
\begin{enumerate}[leftmargin=*]
    \item $B = S \Lambda_B$ and $W_\psi = \Lambda_\psi S^{-1}$, where $\Lambda_B = \operatorname{diag}(\beta_1, \dots, \beta_n)$ and $\Lambda_\psi = \operatorname{diag}(\gamma_1, \dots, \gamma_n)$.
    \item $|\mu_i| + \overline{\Phi'} |\beta_i \gamma_i| < 2\sqrt{\alpha_i}$ for every $i$th eigenvalue, where $\overline{\Phi'} := \max\{|\Phi'_{\min}|, |\Phi'_{\max}|\}$.
\end{enumerate}

Then, for any linear matrix $D_{m-1}^t \in \mathcal{D}$, the SpatialRNN gradient dynamics decouple into $n$ independent second-order modes. For each mode $i$, the associated characteristic equation yields exactly two simple complex conjugate eigenvalues with a time-invariant modulus $|\lambda| = \sqrt{\alpha_i}$.
\end{theorem}

\begin{proof}
Recall that the characteristic equation of the second-order gradient dynamics of the SpatialRNN~\eqref{eq:SpatialRNNgradients} is given by
\begin{equation}
    \det(z^2 \mathbb{I} - z(\overline{W} + B D_{m-1}^t W_\psi) - \overline{\overline{W}}) = 0.
\end{equation}
Substituting the conditions $\overline{W} = S \overline{\Lambda} S^{-1}$, $\overline{\overline{W}} = S \overline{\overline{\Lambda}} S^{-1}$, $B = S \Lambda_B$, and $W_\psi = \Lambda_\psi S^{-1}$ yields
\begin{equation}
    \det(S [z^2 \mathbb{I} - z(\overline{\Lambda} + \Lambda_B D_{m-1}^t \Lambda_\psi) - \overline{\overline{\Lambda}}] S^{-1}) = 0.
\end{equation}
Since $D_{m-1}^t \in \mathcal{D}$ is diagonal, it commutes with the diagonal matrices $\Lambda_B$ and $\Lambda_\psi$. Let $\tilde{D} = \Lambda_B D_{m-1}^t \Lambda_\psi$, which is a diagonal matrix with entries $\tilde{d}_i = \beta_i \gamma_i d_i$, where $d_i = [D_{m-1}^t]_{ii}$ and $|d_i| \le \overline{\Phi'}$. By the similarity invariance of the determinant, the system fully decouples into $n$ independent quadratic equations:
\begin{equation}
    z^2 - (\mu_i + \beta_i \gamma_i d_i)z + \alpha_i = 0, \quad \text{for } i = 1, \dots, n.
\end{equation}
The roots for each mode $i$ are:
\begin{equation}
    z_{i,\pm} = \frac{(\mu_i + \beta_i \gamma_i d_i) \pm \sqrt{(\mu_i + \beta_i \gamma_i d_i)^2 - 4\alpha_i}}{2}.
\end{equation}
To guarantee simple, complex conjugate roots (i.e., underdamped oscillatory modes), the discriminant must be strictly negative: $(\mu_i + \beta_i \gamma_i d_i)^2 < 4\alpha_i$. Applying the triangle inequality and the activation derivative bound, we have that
\begin{equation}
    |\mu_i + \beta_i \gamma_i d_i| \le |\mu_i| + |\beta_i \gamma_i| |d_i| \le |\mu_i| + \overline{\Phi'} |\beta_i \gamma_i|.
\end{equation}
By assumption, $|\mu_i| + \overline{\Phi'} |\beta_i \gamma_i| < 2\sqrt{\alpha_i}$. Thus, the discriminant is strictly negative for any LTV sequence matrix $D_{m-1}^t \in \mathcal{D}$.
Consequently, the roots are complex conjugates, and their squared magnitude evaluates to
\begin{equation}
    |z_{i,\pm}|^2 = \left(\frac{\mu_i + \beta_i \gamma_i d_i}{2}\right)^2 + \frac{4\alpha_i - (\mu_i + \beta_i \gamma_i d_i)^2}{4} = \alpha_i.
\end{equation}
Therefore, $|z_{i,\pm}| = \sqrt{\alpha_i}$, establishing robust eigenvalue placement for all modes invariant to the time-varying sequence.
\end{proof}

\subsection{LTV Stability Conditions for the SpatialRNN Gradient Dynamics}\label{apx:ltv}

The conditions derived in Theorem~\ref{thm:Stable-Neutral-Mixing} ensure marginal stability of the gradient dynamics for any frozen-time matrix $D\in\mathcal D$. However,
given that the input sequences $\{x_t\}_{t\geq0}$ are unknown and therefore $\{D_m^t\}_{m\geq0}$ can switch arbitrarily, marginal stability of the frozen-time systems is only a necessary, but not sufficient, condition for the marginal stability of the switching LTV gradient equations \eqref{eq:SpatialRNNfullrecurrence}.
%
%
Empirically, we observe that gradient norms can exhibit mild drift over long horizons, although deviations are relatively small, as shown in Fig.~\ref{fig:exploding_grads}.

\begin{figure}[h]
    \centering
    \includegraphics[width=0.5\linewidth]{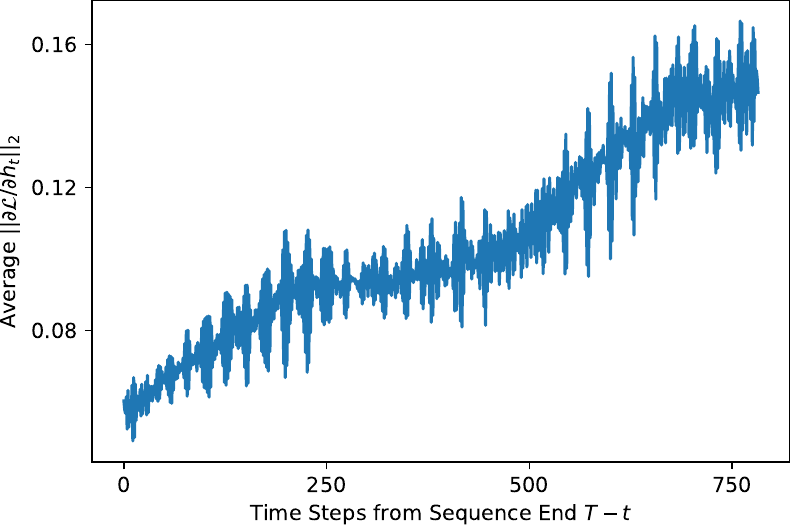}
    \caption{Gradient norm $\| \frac{\partial \mathcal L}{\partial h_{T-t}}\|$ as a function of the time horizon, shown for the psMNIST classification task. Matrix norms are set such that $||B_i||_2 < 0.4, \ ||W_{\psi,i}||_2 < 0.4$ and $||\overline W_{i}||_2 < 1.8$.}
    \label{fig:exploding_grads}
\end{figure}

In order to prevent exploding gradients, we provide further constraints on the SpatialRNN's weight matrices to ensure \textit{asymptotic stability} of the actual LTV gradient dynamics. These results are weaker than Theorem~\ref{thm:Stable-Neutral-Mixing} in the sense that we cannot control the decay rate of the system, nor can we guarantee marginal stability ($\alpha_i=1$), yet they provide concrete bounds on the gradient dynamics.

\begin{theorem}[LTV stability of the SpatialRNN gradients] \label{thm:guas}
Let the matrices $\overline{W}$ and $\overline{\overline{W}}$ be defined as block diagonal matrices with blocks sizes $n_i$, $i=1,\ldots,k$, such that 
\begin{equation}
    \overline{W} = \begin{pmatrix}
            \overline{W}_1 & 0 & 0  \\
            0 & \ddots & 0  \\
            0 & 0 & \overline{W}_k
        \end{pmatrix} \quad \text{and} \quad
        \overline{\overline{W}} = \begin{pmatrix}
            -\alpha_1\mathbb{I}_{n_1} & 0  & 0  \\
            0 & \ddots & 0  \\
            0 & 0 & -\alpha_{k}\mathbb{I}_{n_k}
        \end{pmatrix},
\end{equation}
where $\overline{W}_i\in\mathbb R^{n_i\times n_i}$, $\alpha_i \in (0,1)$, $\sum_{i=1}^k n_i = n$ is the model size.
Suppose the following conditions are satisfied:
    \begin{enumerate}[leftmargin=*]
    \item $\overline{W_i} = \overline{W_i}^\top$, $B = \operatorname{blkdiag}\{B_1,\ldots,B_k\}$, for block matrices $B_i \in \mathbb R^{n_i\times n_i}$, and $W_\psi = \Sigma B^\top = \operatorname{blkdiag}\{W_{\psi,1},\ldots,W_{\psi,k}\}$ for some diagonal matrix $\Sigma \in \operatorname{diag}(\mathbb R^n)$.
    \item $||\overline{W}_i||_2 <  \frac{1+\alpha_i}{\sqrt{\alpha_i}}$ and $\overline{\Phi'} ||B_i||_2 ||W_{\psi,i}||_2 < (1-\alpha_i) \left( \sqrt{1 - \frac{\alpha_i}{(1+\alpha_i)^2} ||\overline{W}_i||_2^2} - \frac{||\overline{W}_i||_2}{1+\alpha_i} \right) \ \forall i$, where $\overline{\Phi'} := \max\{|\Phi'_{\min}|, |\Phi'_{\max}|\}$.
\end{enumerate}
 Then, the LTV gradient dynamics $\Psi_m^t$ are uniformly asymptotically stable for any matrix sequence $\{D_{m}^t\}_{m\geq 0}, \ D_{m}^t \in \mathcal{D}$.
\end{theorem}

\begin{proof}
The block-diagonal structure of the parameter matrices naturally decouples the gradient dynamics into $k$ independent LTV subsystems. For the $i$-th block, the recurrence is given by $x^{(i)}_{m} = A_{m,i} x^{(i)}_{m-1}$, where $x^{(i)}_{m} = \begin{bmatrix} \Psi_{m,i}^t \\ \Psi_{m-1,i}^t \end{bmatrix}$, $A_{m,i} = \begin{bmatrix} W_{m,i} & -\alpha_i I \\ I & 0 \end{bmatrix}$, $W_{m,i} = \overline{W}_i + \tilde{W}_{m,i}$, and $\tilde{W}_{m,i} = B_i D_{m-1}^t W_{\psi,i}$.
By definition, the spectral norm of the perturbation is bounded: $\|\tilde{W}_{m,i}\|_2 \le \tilde{w}_i$ for any $D_{m-1}^t \in \mathcal{D}$.

To guarantee stability under arbitrary sequential variations, we construct a time-invariant quadratic Lyapunov function for each LTV subsystem: $V_i(x) = x^\top P_i x$. If a fixed positive-definite matrix $P_i$ exists such that the energy difference $\Delta V_{i,m} = A_{m,i}^\top P_i A_{m,i} - P_i$ is strictly negative definite at every step, uniform stability is guaranteed~\citep{liberzon2003switching}. We define $P_i$ as
\begin{equation}
    P_i = \begin{bmatrix} I & -\gamma_i \overline{W}_i \\ -\gamma_i \overline{W}_i & \alpha_i I \end{bmatrix}, \quad \text{where} \quad \gamma_i = \frac{\alpha_i}{1+\alpha_i}.
\end{equation}
By the Schur complement, $P_i \succ 0$ if and only if $\alpha_i I - \gamma_i^2 \overline{W}_i^2 \succ 0$. This holds precisely when $\|\overline{W}_i\|_2 < \frac{1+\alpha_i}{\sqrt{\alpha_i}}$, which is imposed by condition 2.

Next, we evaluate the difference $\Delta V_{i,m}$. Expanding $A_{m,i}^\top P_i A_{m,i}$ and substituting $\gamma_i$ forces the linear $\overline{W}_i$ terms in the off-diagonal blocks to perfectly cancel out, yielding:
\begin{equation}
    \Delta V_{i,m} = \begin{bmatrix} S_{11} & -\alpha_i \tilde{W}_{m,i} \\ -\alpha_i \tilde{W}_{m,i} & -\alpha_i(1-\alpha_i) I \end{bmatrix},
\end{equation}
where $S_{11} = \tilde{W}_{m,i}^2 + \frac{1}{1+\alpha_i}(\overline{W}_i \tilde{W}_{m,i} + \tilde{W}_{m,i} \overline{W}_i) + \frac{1-\alpha_i}{1+\alpha_i} \overline{W}_i^2 - (1-\alpha_i) I$.
Since $\alpha_i \in (0,1)$, the bottom-right block $S_{22} = -\alpha_i(1-\alpha_i) I$ is strictly negative definite. Applying the Schur complement to $\Delta V_{i,m}$, we require $S_{11} - S_{12} S_{22}^{-1} S_{12}^\top \prec 0$. Thus, we have that
\begin{equation}
    \frac{1}{1-\alpha_i}\tilde{W}_{m,i}^2 + \frac{1}{1+\alpha_i}(\overline{W}_i \tilde{W}_{m,i} + \tilde{W}_{m,i} \overline{W}_i) + \frac{1-\alpha_i}{1+\alpha_i}\overline{W}_i^2 - (1-\alpha_i) I \prec 0.
\end{equation}
To enforce strict negative definiteness, we bound the spectral norm. Using the triangle inequality and the property $\|\overline{W}_i \tilde{W}_{m,i} + \tilde{W}_{m,i} \overline{W}_i\|_2 \le 2 \|\overline{W}_i\|_2 \|\tilde{W}_{m,i}\|_2$, we obtain the sufficient scalar inequality
\begin{equation}\label{eq.proof.quadraticeq}
    \frac{1}{1-\alpha_i}\|\tilde{W}_{m,i}\|_2^2 + \frac{2\|\overline{W}_i\|_2}{1+\alpha_i}\|\tilde{W}_{m,i}\|_2 + \frac{1-\alpha_i}{1+\alpha_i}\|\overline{W}_i\|_2^2 - (1-\alpha_i) < 0.
\end{equation}
Multiplying by $(1-\alpha_i)$, we can solve this quadratic inequality for $\|\tilde{W}_{m,i}\|_2$. The following strict upper bound, corresponding to the positive root of Eq.~\eqref{eq.proof.quadraticeq}, guarantees stability:
\begin{equation}
    \|\tilde{W}_{m,i}\|_2 < (1-\alpha_i) \left( \sqrt{1 - \frac{\alpha_i}{(1+\alpha_i)^2}\|\overline{W}_i\|_2^2} - \frac{\|\overline{W}_i\|_2}{1+\alpha_i} \right).
\end{equation}
Condition 2 guarantees this inequality holds, ensuring $\Delta V_{i,m} \prec 0$ uniformly over time. Summing the block-wise Lyapunov functions, $V(x) = \sum_{i=1}^k V_i(x^{(i)})$ forms a global time-invariant Lyapunov function, establishing uniform asymptotic stability for the full LTV system.
\end{proof}

While condition 1 can be relaxed to the block-triangular structure considered in Theorem~\ref{thm:Stable-Neutral-Mixing}, the most significant change is the bound on the matrix norms. From Theorem~\ref{thm:Stable-Neutral-Mixing}, to ensure marginal stability we require $\alpha_i\to1$. However, due to condition 2 of Theorem~\ref{thm:guas}, we would need to impose the constraint $\|B_i\|_2 \|W_{\psi,i}\|_2 \to 0$ for the state-to-medium and medium-to-state matrices. Such constraint would render the SpatialRNN untrainable. Thus, instead we enforce these matrices to have small (but nonvanishing) norms, which completely prevents exploding gradients at the cost of introducing some degree of gradient decay. Consequently, under Theorem \ref{thm:guas}, the LTV gradient dynamics of the SpatialRNN are asymptotically stable, but not necessarily marginally stable.

\textbf{Input-independent volume contraction.}
Although Theorem~\ref{thm:Stable-Neutral-Mixing} is a frozen-time result, part of the structure that it exploits carries over exactly to the LTV gradient dynamics. Since the input-dependent term $B D_{m-1}^t W_\psi$ only enters the first block of the companion matrix in Eq.~\eqref{eq:SpatialRNNgradients}, the determinant of the gradient transition matrix does not depend on the input sequence.

\smallskip
\begin{proposition}[Input-independent volume contraction]\label{prop:det}
Let $\overline{\overline W}$ be defined as in Theorem~\ref{thm:Stable-Neutral-Mixing}, $\mathcal A_m$ denote the companion matrix of the gradient dynamics~\eqref{eq:SpatialRNNgradients} at step $m$, and $\Phi_m := \mathcal A_m\mathcal A_{m-1}\cdots\mathcal A_2$. Then, for any sequence $\{D_m^t\}_{m\geq1}$, with $D_m^t\in\mathcal D$, it holds that
\begin{equation}
    |\det\Phi_m| = \Big(\prod\nolimits_{i=1}^k\alpha_i^{n_i}\Big)^{m-1}, \quad m\geq2.
\end{equation}
In particular, if $\alpha_i=1$ for all $i=1,\ldots,k$, then $\|\Phi_m\|_2\geq 1$ for all $m\geq2$.
\end{proposition}
\begin{proof}
Exchanging the two block rows of $\mathcal A_m$, each of size $n$, yields a block lower-triangular matrix with diagonal blocks $\mathbb I$ and $\overline{\overline W}$. Since this exchange is a permutation with sign $(-1)^{n}$, and $\overline{\overline W}$ is block upper-triangular with diagonal blocks $-\alpha_i\mathbb I_{n_i}$, it follows that $\det\mathcal A_m = (-1)^n\det\overline{\overline W} = \prod_{i=1}^k\alpha_i^{n_i}$, independently of $D_{m-1}^t$. The first claim then follows from the multiplicativity of the determinant. For the second claim, note that $|\det\Phi_m|$ equals the product of the singular values of $\Phi_m$; if $|\det\Phi_m|=1$, then its largest singular value satisfies $\|\Phi_m\|_2\geq1$.
\end{proof}

Consequently, for $\alpha_i=1$, the LTV gradient dynamics of the SpatialRNN are volume preserving for any input sequence, and the medium gradients cannot vanish along all directions. In contrast, the companion matrix of the HORNN gradient dynamics~\eqref{eq:HORNNgrad} has determinant $\pm\det(D_m^t)\det(W_d)$, which depends on the input and vanishes as the activations saturate. Volume preservation is, however, necessary but not sufficient for marginal stability, as it does not preclude growth along some directions compensated by decay along others.

\subsection{Empirical LTV Gradient Dynamics along Input Trajectories}\label{apx:ltv_empirical}

To quantify the gap between the frozen-time analysis and the full LTV gradient dynamics, we compute the exact state-to-state Jacobians $G^\tau_m = \partial h_{\tau+m}/\partial h_\tau$ of a SpatialRNN along real input trajectories. To this end, we propagate the recurrence equations~\eqref{eq:SpatialRNNfullrecurrence}--\eqref{eq:SpatialRNNgradients} using the activation derivatives $D^\tau_m$ realized by each input sequence. In contrast with the loss gradients in Fig.~\ref{fig:exploding_grads}, this measurement does not depend on the loss, the readout, or the label, and thus directly probes the LTV dynamics whose frozen-time spectrum is controlled by Theorem~\ref{thm:Stable-Neutral-Mixing}.
We consider a SpatialRNN with block sizes $\{128, 8, 8, 8, 8\}$ trained on psMNIST, with $\alpha_1=1$ fixed and learned decay parameters $\alpha_i\in[0.31, 0.46]$ for $i\geq2$ (which reached $96.2\%$ validation accuracy during training), and the same architecture at initialization (with $\alpha_i = 0.2$ for $i\geq 2$). For each model, we evaluate $\|G^\tau_m\|_F$ for $\tau=0$ and $m\leq 783$ along $100$ psMNIST validation sequences in double precision, rescaling the recursion at every step to prevent overflow and underflow. For each sequence, we fit a log-linear regression to $\|G^\tau_m\|_F$ for $m\geq50$, whose slope provides an empirical estimate of the top Lyapunov exponent of the LTV gradient dynamics.

Fig.~\ref{fig:ltv_jacobians} and Table~\ref{tab:ltv} show that the Jacobian norms neither vanish nor explode over the full horizon. For the trained model, the average slope is $7.98\cdot10^{-5}$ per step, corresponding to a change of only ${\approx}6\%$ in the Jacobian norm across $783$ steps, compared with ${\approx}23\%$ at initialization. Training thus moves the model closer to the marginal regime predicted by the frozen-time analysis and reduces the residual drift due to the time-varying dynamics.

\begin{figure}[h]
    \centering
    \includegraphics[width=\linewidth]{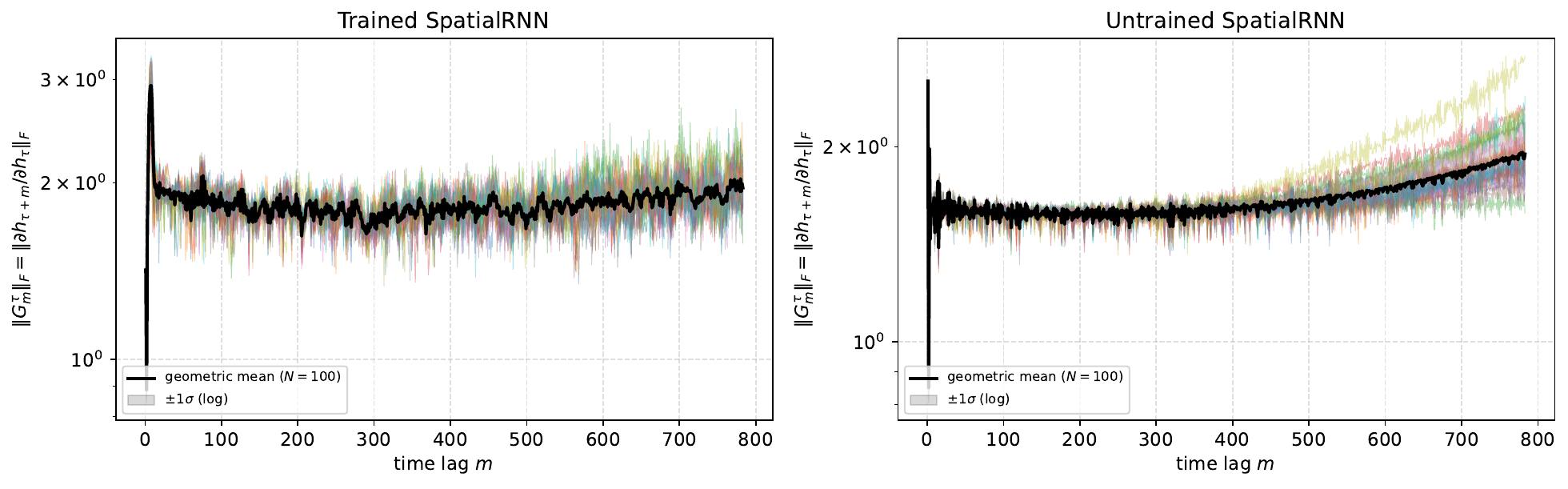}
    \caption{Frobenius norm of the exact state-to-state Jacobians $G^\tau_m=\partial h_{\tau+m}/\partial h_\tau$, propagated via Eqs.~\eqref{eq:SpatialRNNfullrecurrence}--\eqref{eq:SpatialRNNgradients} along $100$ psMNIST validation sequences ($\tau=0$, $T=784$), for a trained SpatialRNN (left) and the same architecture at initialization (right). Thin lines show individual sequences, while the black line and the shaded band show the geometric mean and $\pm1$ standard deviation (in log scale).}
    \label{fig:ltv_jacobians}
\end{figure}

\begin{table}[h]
\centering
\begin{tabular}{lccc}
\toprule
Model & Log-linear slope (per step) & Change over $783$ steps & CV of $\|G^\tau_m\|_F$ \\
\midrule
Trained SpatialRNN   & $(7.98\pm7.75)\cdot10^{-5}$ & ${\approx}6\%$  & $0.071\pm0.010$ \\
Untrained SpatialRNN & $(2.60\pm1.51)\cdot10^{-4}$ & ${\approx}23\%$ & $0.072\pm0.032$ \\
\bottomrule
\end{tabular}
\smallskip
\caption{Empirical LTV gradient statistics over $100$ psMNIST validation sequences (mean $\pm$ standard deviation across sequences). Slopes are fitted for $m\geq50$, and the coefficient of variation (CV) measures the dispersion of $\|G^\tau_m\|_F$ along each sequence.}
\label{tab:ltv}
\end{table}

\section{Symplectic Form Preservation of the SpatialRNN} \label{apx:symplectic}

When the weight matrices are constrained to ensure marginal stability ($\alpha_i=1$), the corresponding gradient dynamics can maintain a symplectic form~\citep{marsden2001discrete}.
%
%
A map is symplectic if it preserves the canonical geometry of phase space (i.e., it conserves a bilinear form that encodes position-momentum interactions) and therefore preserves the phase-space volume.
Unlike dissipative dynamics, symplectic dynamics preserve phase-space volume and are invertible, so information is not irreversibly lost through volume contraction.
Consequently, this property leads to numerically stable and time-reversible evolution, allowing past states to be recovered exactly from future ones (as shown below in Theorem~\ref{thm:state_reconstruction}). Here, we first formalize the condition for symplectic form preservation:

\begin{proposition}[Symplectic preservation of second-order LDEs]\label{prop:symp-form}
Consider the unforced medium dynamics 
\begin{equation}\label{eq:inputless_pde}
    \psi_{t+1} = \overline{W}\psi_t + \overline{\overline{W}}\psi_{t-1},
\end{equation}
and define the canonical position-momentum coordinates $z_t = (q_t, p_t)^T$, where $q_t = \psi_t$ and $p_t = \frac{1}{\Delta t}(\psi_t - \psi_{t-1})$. In these coordinates, the evolution defines a one-step linear map $z_{t+1} = M z_t$. This map is symplectic (i.e., $M^\top J M = J$ for the canonical bilinear form $J$) if and only if
$\overline{\overline{W}} = -\mathbb I$ and $\overline{W}^\top = \overline{W}$.
\end{proposition}

\begin{proof}
The one-step map $M$ is derived from the canonical coordinate definitions:
\begin{align*}
    q_{t+1} &= \psi_{t+1} = \overline{W}q_t + \overline{\overline{W}}(q_t - {\Delta t} p_t) = (\overline{W} + \overline{\overline{W}})q_t - {\Delta t} \overline{\overline{W}} p_t, \\
    p_{t+1} &= {\Delta t}^{-1}(q_{t+1} - q_t) = \frac{1}{\Delta t}(\overline{W} + \overline{\overline{W}} - \mathbb{I})q_t - \overline{\overline{W}} p_t.
\end{align*}
Therefore, the dynamics can be written as a linear map $z_{t+1} = Mz_t$, with the propagator
\begin{equation}
    M = \begin{pmatrix} \overline{W} + \overline{\overline{W}} & -{\Delta t} \overline{\overline{W}} \\ \frac{1}{\Delta t}(\overline{W} + \overline{\overline{W}} - \mathbb{I}) & -\overline{\overline{W}} \end{pmatrix} \equiv \begin{pmatrix} A & B \\ C & D \end{pmatrix}.
\end{equation}
The map is symplectic if $M^\top J M = J = \begin{pmatrix} 0 & \mathbb{I} \\ -\mathbb{I} & 0 \end{pmatrix}$. This requires four block-matrix conditions: 1) $A^\top D - C^\top B = \mathbb{I}$, 2) $A^\top C - C^\top A = 0$, 3) $B^\top D - D^\top B = 0$, and $B^\top C - D^\top A = -\mathbb{I}$.

\smallskip\noindent
\textit{Condition 1}. Substituting the expressions for $(A,B,C,D)$, we have that
    \begin{align*}
    (\overline{W} + \overline{\overline{W}})^\top (-\overline{\overline{W}}) - (\frac{1}{\Delta t}(\overline{W} + \overline{\overline{W}} - \mathbb{I}))^\top (-{\Delta t} \overline{\overline{W}}) 
    &= -(\overline{W}^\top + \overline{\overline{W}}^\top) \overline{\overline{W}} + (\overline{W}^\top + \overline{\overline{W}}^\top - \mathbb{I}) \overline{\overline{W}} \\
    &= (-\overline{W}^\top - \overline{\overline{W}}^\top + \overline{W}^\top + \overline{\overline{W}}^\top - \mathbb{I}) \overline{\overline{W}} \\
    &= -\overline{\overline{W}},
    \end{align*}
    implying $\overline{\overline{W}} = -\mathbb{I}$.

\smallskip\noindent
\textit{Condition 2}. Using $\overline{\overline{W}} = -\mathbb{I}$ from condition 1, the block matrices simplify to $A = \overline{W} - \mathbb{I}$ and $C = \frac{1}{\Delta t}(\overline{W} - 2\mathbb{I})$, yielding the following expression for this condition:
    \begin{align*}
    (\overline W - \mathbb{I})^\top \frac{1}{\Delta t}(\overline{W} - 2\mathbb{I}) &- \frac{1}{\Delta t}(\overline{W} - 2\mathbb{I})^\top (\overline{W} - \mathbb{I}) \\
    &= \frac{1}{\Delta t} \left[ (\overline{W}^\top - \mathbb{I})(\overline{W} - 2\mathbb{I}) - (\overline{W}^\top - 2\mathbb{I})(\overline{W} - \mathbb{I}) \right] \\
    &= \frac{1}{\Delta t} [ -\overline{W}^\top + \overline{W} ] \\&= 0,
    \end{align*}
    implying $\overline{W}^\top = \overline{W}$.

\medskip
Finally, conditions 3 and 4 are automatically satisfied by these constraints (with condition 4 being the transpose counterpart of condition 1).
\end{proof}

\begin{remark}
    The wave equation (presented in Remark~\ref{rmk:stormerverlet}) yields the spatial update matrices $\overline{W} = 2\mathbb{I} + c^2 \Delta t^2 W_\Delta$ and $\overline{\overline{W}} = -\mathbb{I}$. Because the discrete Laplacian $W_\Delta$ is symmetric, $\overline{W}^\top = \overline{W}$, which directly satisfies Proposition~\ref{prop:symp-form} and preserves the symplectic form.
\end{remark}

\subsection{Exact Backward Reconstruction of the SpatialRNN States}
We can now formalize the structural conditions under which the SpatialRNN admits $O(1)$ memory cost during training. In contrast to standard RNNs, where the intermediate states must be stored explicitly for backpropagation, the SpatialRNN enables exact reconstruction of its full trajectory from the final state. This property follows from the second-order dynamics of the medium, which induce a reversible state evolution. When the medium is non-dissipative ($\overline{\overline W}=-\mathbb I$ and $\overline W=\overline W^\top$), the forward dynamics preserve a symplectic structure (Proposition~\ref{prop:symp-form}), implying non-dissipative information flow and time-reversibility. In particular, for blocks with $\alpha_i = 1$, no information is lost over time, while for $\alpha_i<1$ the invertibility of the medium dynamics still permits, in exact arithmetic, the recovery of past states despite dissipative behavior. As a result, intermediate states can be reconstructed by evolving the system backward in time, eliminating the need to cache activations during the forward pass. The following theorem explicitly constructs this backward mapping.

\smallskip
\begin{theorem}[Exact history reconstruction] \label{thm:state_reconstruction}
Suppose the weight matrices satisfy the conditions of Theorem~\ref{thm:Stable-Neutral-Mixing}. Then, given the input sequence $\{x_t\}_{t=1}^T$, the complete trajectory $\{\psi_t, h_t\}_{t=1}^T$ can be exactly reconstructed by iteratively recovering past states starting from the final boundary states $(\psi_T, \psi_{T-1})$.
\end{theorem}

\begin{proof}
Assume that during the backward pass at step $t$, the states $\psi_t$ and $\psi_{t-1}$ are available in memory, along with the given input sequence $\{x_t\}_{t=1}^T$. We show that this information suffices to uniquely reconstruct the preceding states. First, we note that the hidden state at time $t-1$ is perfectly reconstructed from the forward update equation:
\begin{equation}
    h_{t-1} = \Phi(W_x x_{t-1} + W_\psi \psi_{t-1} + b_h)
\end{equation}

Next, isolating $\psi_{t-2}$ in the medium dynamics yields
\begin{equation}
    \overline{\overline{W}} \psi_{t-2} = \psi_t - \overline{W} \psi_{t-1} - B h_{t-1}
\end{equation}
Under the assumption of Theorem~\ref{thm:Stable-Neutral-Mixing}, we have that $\exists \ \overline{\overline{W}}^{-1}$. Applying its inverse gives a unique solution for the prior medium state:
\begin{equation}
    \psi_{t-2} = \overline{\overline{W}}^{-1} \left(\psi_t - \overline{W} \psi_{t-1} - B h_{t-1}\right).
\end{equation}
This establishes that the pair $(\psi_t, \psi_{t-1})$ uniquely determines $(\psi_{t-1},\psi_{t-2})$, and hence the state at time $t-2$.
By recursively applying this procedure from $t = T$ down to $t=2$, we recover the complete forward sequence $\{\psi_t, h_t\}_{t=1}^T$, without requiring storage of intermediate states.
\end{proof}

\subsection{Numerical Precision of the Reconstruction}
Theorem~\ref{thm:state_reconstruction} holds in exact arithmetic. In finite precision, round-off is propagated by the backward recursion, whose conditioning is governed by the inverse medium dynamics. Under the conditions of Theorem~\ref{thm:Stable-Neutral-Mixing} the modes of the medium
have modulus $\sqrt{\alpha_i}$, so the backward recursion, which inverts them, amplifies round-off by $\alpha_i^{-1/2}$ per step. For a tolerance $\varepsilon_{\mathrm{tol}}$ and machine epsilon $\varepsilon_{\mathrm{mach}}$, the useful horizon is therefore
\begin{equation}\label{eq:reversible-horizon}
  m_{\max}\;\approx\;
  \frac{\log\!\big(\varepsilon_{\mathrm{tol}}/\varepsilon_{\mathrm{mach}}\big)}
       {\log\!\big(\alpha_{\min}^{-1/2}\big)},
  \qquad \alpha_{\min}=\min_i \alpha_i .
\end{equation}
This is a property of the recursion rather than of the data, so we isolate it in a controlled setting: we compare the Jacobians $G_m=\partial h_{T-1}/\partial h_{T-1-m}$ obtained from the stored forward states with those obtained from the backward-reconstructed states, for randomly initialized models driven by $20$ synthetic scalar input sequences of length $T=784$ (Fig.~\ref{fig:reversibility}). In the non-dissipative case ($\alpha_i=1\;\forall i$) the recursion has all its modes on the unit circle, no amplification occurs, and the relative error stays below $10^{-5}$ in single and $10^{-13}$ in double precision over the entire horizon. With dissipative blocks ($\alpha=[1,0.5,0.5,0.5]$) the error grows exponentially at a measured rate of $1.42$ per step, matching the $1/\sqrt{0.5}=1.414$ predicted by the modal analysis, and exceeds $10^{-3}$ after $44$ steps in single and $101$ steps in double precision, in agreement with~\eqref{eq:reversible-horizon}. Because $\overline{W}$ and $\overline{\overline{W}}$ are block \emph{upper}-triangular, the marginally stable block is not shielded: it reads from the dissipative blocks and inherits their divergence. Consequently, $\mathcal{O}(1)$-memory BPTT is practical only when all blocks are non-dissipative; models with $\alpha_i<1$ require caching, or checkpointing every $K<m_{\max}$ steps, which brings the memory use to $\mathcal{O}(T/K)$, and could be used for parallelization.

\begin{figure}[h]
    \centering
    \begin{minipage}{0.49\linewidth}
        \centering
        {\small (\textbf{a}) Single block, $\alpha_1=1$}\\[2pt]
        \includegraphics[width=\linewidth,trim={0 0 0 69bp},clip]{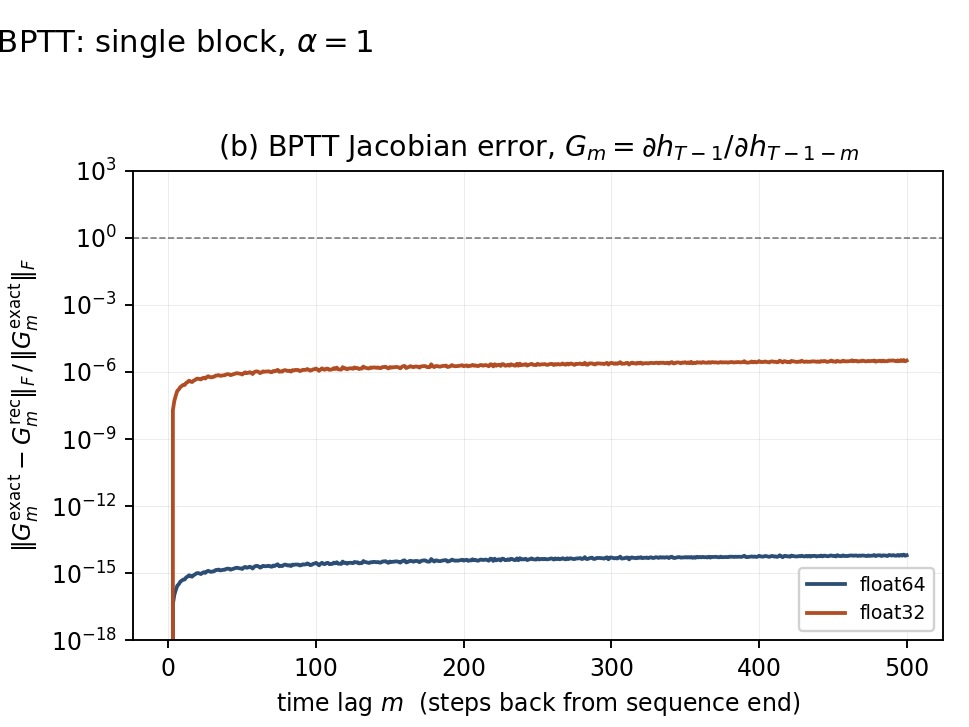}
    \end{minipage}\hfill
    \begin{minipage}{0.49\linewidth}
        \centering
        {\small (\textbf{b}) Four blocks, $\alpha=[1,0.5,0.5,0.5]$}\\[2pt]
        \includegraphics[width=\linewidth,trim={0 0 0 69bp},clip]{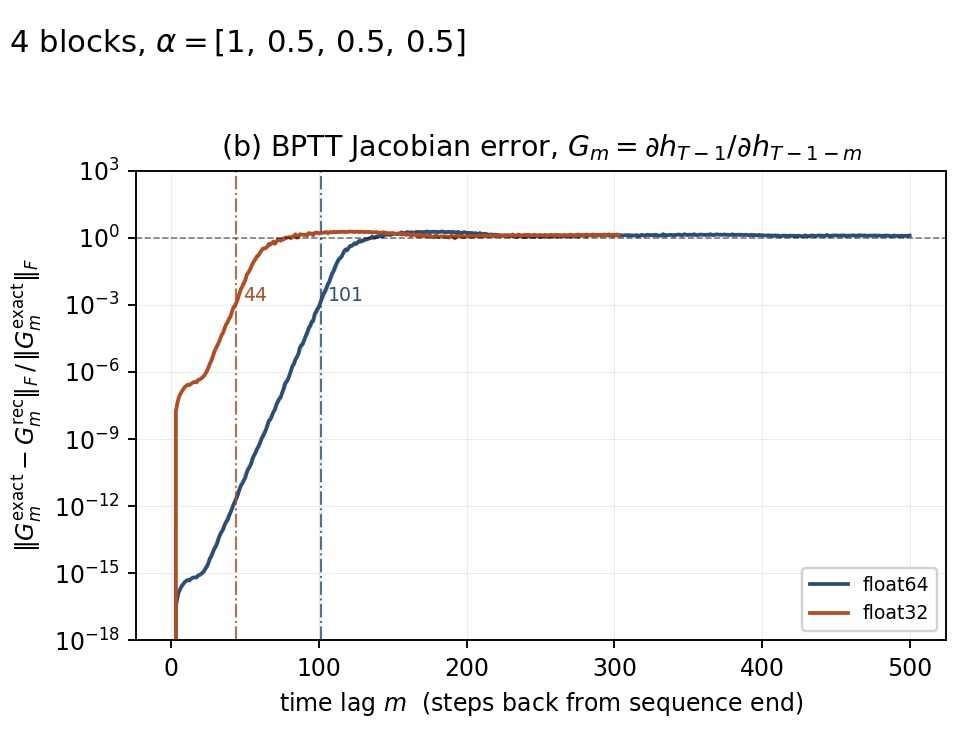}
    \end{minipage}
    \caption{Relative Frobenius error
$\|G_m^{\mathrm{exact}}-G_m^{\mathrm{rec}}\|_F/\|G_m^{\mathrm{exact}}\|_F$
between the exact Jacobians $G_m=\partial h_{T-1}/\partial h_{T-1-m}$, computed from the stored forward states, and those computed from the backward-reconstructed states, as a function of the time lag $m$, in single (float32) and double (float64) precision. \textbf{(a)} Non-dissipative SpatialRNN (single block, $n=40$, $\alpha_1=1$): the error remains at machine precision over the entire horizon, with no exponential accumulation. \textbf{(b)} SpatialRNN with dissipative blocks (block sizes $\{22,6,6,6\}$, $\alpha=[1,0.5,0.5,0.5]$): round-off is amplified exponentially, at a measured rate of $1.42$ per backward step, matching the predicted $\alpha_{\min}^{-1/2}=1.414$. 
Vertical lines mark the lags at which the median error first exceeds $10^{-3}$, yielding $m=44$ in single  and $m=101$ in double precision.
Since the reconstruction starts from $(\psi_{T-1},\psi_{T-2})$ and proceeds backward, the state at time $t$ is recovered after $T-1-t$ backward steps, so $m$ is also the depth of the backward recursion. Curves show the median over $20$ input sequences of length $T=784$ and shaded bands the $10$--$90$ percentile range; the same effective matrices are used in both precisions, so the two curves differ only in the arithmetic precision of the recursions. Both models are randomly initialized, with the diagonal blocks rescaled so that condition~2 of Theorem~\ref{thm:Stable-Neutral-Mixing} holds. }
    \label{fig:reversibility}
\end{figure}

\section{Experimental Details}\label{apx:exp_det}

In this section, we detail the datasets, training procedures, and benchmarking used in Section~\ref{sec:experiments}, along with representative data samples. All code for reproduction is available at: [\url{anonymous.4open.science/r/SpatialRNN-code-748C}]
%
%

\subsection{SpatialRNN implementation details}

\textbf{Activation function and decay parameters.}
To implement the SpatialRNN in practice, we set the activation function to $\Phi = \tanh$. In all experiments, we enforce the structural and spectral conditions outlined in Theorem~\ref{thm:Stable-Neutral-Mixing} to guarantee the retention of long-term dependencies. We thus set $\alpha_1 = 1$ (typically corresponding to the largest block), while treating the remaining decay parameters $\{\alpha_2,\alpha_3,\ldots\}$ as trainable.

\textbf{Masking and initialization.}
To satisfy the block upper-triangular constraints of Theorem~\ref{thm:Stable-Neutral-Mixing}, we apply static binary masks during each forward pass. We use block-diagonal masks for $B$, strictly diagonal mask for $\Sigma$, set $W_\psi = \Sigma B^\top$, and use block triangular masks for both $\bar{W}$ and $\bar{\bar{W}}$.

The input and output projection matrices ($W_x, W_y$) are initialized using Xavier uniform distribution. For the medium dynamics, the block-diagonal components of $\bar{W}$ are initialized using orthogonal initialization (gain of $0.5$), while $B$ and $\Sigma$ are initialized as the identity matrix. The upper-triangular interaction blocks are initialized with Xavier uniform distribution with a reduced gain of $0.1$ to not dominate the intra-block dynamics.

\textbf{Spectral Bounding.} To strictly satisfy the stability bounds at every iteration, we dynamically project the weight matrices using their spectral norms. Following Theorem~\ref{thm:Stable-Neutral-Mixing}, we apply a multiplicative scaling factor $\gamma_i$ to the blocks to constrain the gradient spectrum:
\begin{equation}\label{eq:common_bound}
    \gamma_i = \min\left(1, \frac{2\sqrt{\alpha_i} - \epsilon}{\|\bar{W}_i\|_2 + \|B_i\|_2\|W_{\psi, i}\|_2}\right),
\end{equation}
where $\epsilon = 10^{-4}$ ensures the inequality remains strict. 
Alternatively, we can also use a decoupled bounding scheme where the norms are independently constrained via separate scaling factors (e.g., $\|\bar{W}_i\|_2 \le c_W$ and $\|B_i\|_2, \|W_{\psi, i}\|_2 \le c_B$). Such constraints tend to yield improved performance, as it more directly enforces\textemdash or heuristically approximates\textemdash the LTV stability conditions derived in Theorem~\ref{thm:guas}.

\subsection{HORNN implementation details}
The $d$-order HORNN with delay matrices $\{W_k\}_{k=1}^d$ (as presented in Eq.~\eqref{eq:HORNN}) can be generalized to use a sparse set of delays $\{W_k\}_\mathcal{T}$ for $\mathcal{T} \subset \{1,\ldots,d\}$. This generalization leads to the update equation:
\begin{equation}
    h_t = \Phi\left(W_x x_t + \sum_{k \in \mathcal{T}} W_k h_{t-k} + b_h\right).
\end{equation}
 Such model can more efficiently evaluate long-range memory retention without the parameter overhead of dense delays. The effective temporal receptive field of the model is bounded by $d_{\max} = \max(\mathcal{T})$.

 \textbf{Implementation and initialization.}
Consistent with the SpatialRNN architecture, we take $\Phi = \tanh$. To ensure stable variance scaling across the sparse temporal connections, the input projection $W_x$, output projection $W_y$, and all sparse recurrence matrices $\{W_k\}_{k \in \mathcal{T}}$ are initialized using a Xavier uniform distribution. The hidden state bias $b_h$ and output bias $b_y$ are strictly initialized to zero.

\subsection{Copy Task}
\textbf{Dataset generation.}
The copy task evaluates the model's ability to perfectly retain information over extended temporal horizons. We define a token dimension $D=4$ and a vocabulary of $V = 2^D = 16$ tokens. The data tokens are embedded as vectors in $\{-1, 1\}^D$.

The input sequence $X$ of total length $T$ is constructed as a concatenation of three phases: $X = [X_{\text{data}}, X_{\text{dist}}, X_{\text{trig}}]$, where $X_{\text{data}} \in \{-1, 1\}^{D \times L}$ consists of $L=5$ tokens sampled uniformly at random from the data tokens, $X_{\text{dist}} \in \{0\}^{D \times T_{\text{dist}}}$ consists of $T_{\text{dist}}$ distractor tokens padded with zeros, and $X_{\text{trig}} \in \{1\}^{D \times L}$ consists of $L=5$ trigger tokens padded with ones.

\textbf{Training and evaluation.}
Models are trained to minimize the cross-entropy loss over the final $L$ retrieval steps, comparing the generated outputs (after the whole distractor sequence has been fed into the model) to the initial $L$ data tokens $X_{\text{data}}$. The outputs of the models for the first $L+T_\text{dist}$ are discarded. At each epoch, we generate $N_{\text{train}} = 3000$ training sequences, $N_{\text{test}} = 800$ testing and $N_{\text{val}} = 800$ validation sequences. We take a batch size of $64$, and use the Adam optimizer with a learning rate of $0.002$. Training proceeds for a maximum of $3000$ epochs, using an early stopping patience of $300$ epochs.

\textbf{Model details.}
For the SpatialRNN models, the temporal decay parameters are initialized to $\alpha_i = 0.8$ for all blocks $i$, except for the first block which is strictly fixed to $\alpha_1 = 1.0$ throughout training to ensure at least one non-dissipative, marginally stable memory stream. We performed a hyperparameter search over various block size configurations. Table~\ref{tab:copy_task_hyperparams} details the specific sizes evaluated, their corresponding parameter counts, and the resulting accuracy on the full $T=1024$ sequence length. We bound the matrix norms by $\|\bar{W}_i\|_2 \le 1.7$ and $\|B_i\|_2, \|W_{\psi, i}\|_2 \le 0.2$.
\begin{table}[h]
\centering
\begin{tabular}{lccc}
\toprule
Model & Neurons & Parameters & Mean Accuracy(\%) $\pm$ std. \\
\midrule
SpatialRNN {[32, 8, 8, 8]}    & 56 & 6,815  & 99.04$\pm$ 0.07 \\
SpatialRNN {[8, 8, 8, 8]}     & 32 & 2,915  & 96.90$\pm$ 4.34 \\
SpatialRNN {[4, 4, 4, 4]}     & 16 & 1,179  & 69.33$\pm$ 16.54 \\
\midrule
HORNN $(\mathcal T=[1, 5, 10, 15])$       & 64 & 17,744 & 13.39$\pm$ 3.07 \\
HORNN $(\mathcal T=[1, 5, 10])$         & 64 & 13,648 & 14.37$\pm$ 2.72 \\
HORNN $(\mathcal T=[1, 5])$          & 64 & 9,552  & 11.05 $\pm$ 1.20 \\
Standard RNN ($d=1$) & 64 & 5,456  & 12.81$\pm$ 1.72 \\
\bottomrule
\end{tabular}
\smallskip
\caption{Copy task architecture configurations and performance. Mean accuracy is evaluated at the maximum sequence length $T=1024$. Average and standard deviation over 10 runs.}
\label{tab:copy_task_hyperparams}
\end{table}
\smallskip
\subsection{Permuted Sequential MNIST (psMNIST)}

\textbf{Dataset and preprocessing.}
The psMNIST task takes MNIST images and flattens them into sequences of length $T = 784$ with feature dimension $D_{\text{in}} = 1$. To ablate local spatial structure, a fixed random permutation $\pi \in S_{784}$ (seed 42) is applied to the time steps of all sequences, isolating the task's reliance on long-range temporal dependencies. The vector values are then normalized via $x \mapsto \frac{x - \mu}{\sigma}$ ($\mu = 0.1307, \sigma = 0.3081$).

\textbf{Training and hyperparameters.}
The images are fed to the model sequentially, and after all the sequence has been processed, a single integer (encoded as the maximum one-hot vector) $D_{\text{out}} = 10$ is generated. Models are trained to minimize cross-entropy loss over $N_{\text{runs}} = 3$ independent initializations. Optimization utilizes Adam with a learning rate of $10^{-4}$ and batch size of $128$. Training proceeds for a maximum of $1000$ epochs with an early stopping patience of $300$ epochs. We use the standard $N_{\text{train}} = 60,000$ training images and  $N_{\text{test}} = 10,000$ testing images. 

Temporal decay parameters are initialized to $\alpha_i = 0.8$ for all blocks $i$, except for the first block which is strictly fixed to $\alpha_1 = 1.0$. We bound the matrix norms by $\|\bar{W}_i\|_2 \le 1.7$ and $\|B_i\|_2, \|W_{\psi, i}\|_2 \le 0.2$ and clip the global gradient norm to 1.0. To produce the results for Table~\ref{tab:psmnist_results}, we performed a simple hyperparameter search over various block size configurations. Table \ref{tab:psmnist_sweep} details the specific sizes evaluated, their corresponding parameter counts, and the resulting mean accuracy across 3 independent runs, ordered by performance.

\begin{table}[h]
\centering
\begin{tabular}{lccc}
\toprule
Model & Neurons & Parameters & Mean Accuracy \\
\midrule
SpatialRNN [128, 8, 8, 8] & 152 & 36,497 & $96.34 \pm 0.09$ \\
SpatialRNN [64, 8, 8, 8] & 88 & 12,849 & $95.49 \pm 0.22$ \\
SpatialRNN [32, 8, 8, 8] & 56 & 5,633 & $94.37 \pm 0.13$ \\
\bottomrule
\end{tabular}
\smallskip
\caption{psMNIST task architecture configurations and performance. Mean accuracy is evaluated over 3 runs, ordered by mean accuracy.}
\label{tab:psmnist_sweep}
\end{table}

\textbf{Compared models.}
We evaluate our model against state-of-the-art RNN architectures designed for long-range dependencies. When previous methods (e.g., NWM, coRNN) have been evaluated multiple times in the literature, we prioritize results derived from multi-run aggregations to ensure statistical robustness. Additionally, to guarantee a fair comparison, we restrict all baselines to single-layer models trained under standard optimization protocols. By omitting complex learning rate schedules and regularizers, such as dropout~\citep{srivastava2014dropout}, batch normalization~\citep{ioffe2015batch} or step decay~\citep{he2016deep}, we isolate the intrinsic memory capacity of the recurrent mechanism from the performance gains typically afforded by increased network depth or training enhancements. 

\subsection{Noise Padded CIFAR-10 (npCIFAR10)}

\textbf{Dataset and preprocessing.}
The npCIFAR10 task evaluates the model's ability to extract spatial features and retain them over prolonged noisy horizons. We use the standard CIFAR-10 dataset ($50,000$ training and $10,000$ testing images on 10 classes). Each image is normalized using the channel-wise means $(0.4914, 0.4822, 0.4465)$ and deviations $(0.2023, 0.1994, 0.2010)$, so the entries are scaled approximately to $[-1, 1]$.

The normalized $3 \times 32 \times 32$ images are then permuted and reshaped into sequences of length $L=32$ with feature dimension $D_{in} = 32 \times 3 = 96$. In this representation, each timestep $t \in \{1, \dots, 32\}$ corresponds to an image row across all color channels. To introduce long-term memory requirement, we append random vectors following a standard normal distribution $\mathcal{N}(0, I_{96})$ for the remaining $T - 32$ timesteps, yielding a final sequence length $T$.

\textbf{Training and hyperparameters.}
The models' first $T-1$ outputs are ignored, and the final output $y_T \in \{0,1\}^{10}$ (after processing the complete input sequence of length $T$) is compared to the image class using the standard cross-entropy loss. While the primary benchmark results (Table~\ref{tab:npcifar10_results}) evaluate the standard $T=1000$ regime, we also perform an ablation across varying sequence lengths $T \in \{50, 100, 200, 500, 1000\}$ to show that the amount of noise distractors does not affect the task performance, but the limitation comes from the classification performance of the model.
Models are trained using the Adam optimizer with a learning rate of $10^{-4}$, global gradient clipping of 1.0 and a batch size of $256$. Training is executed for a maximum of $3000$ epochs across $N_{runs}=3$ independent initializations. We employ an early stopping patience of $300$ epochs based on validation performance.

\textbf{Model details.}
For the SpatialRNN models, the temporal decay parameters are initialized to $\alpha_i = 0.8$ for all blocks $i$, except for the first block which is strictly fixed to $\alpha_1 = 1.0$ throughout training to ensure at least one non-dissipative, marginally stable memory stream. The specific model configurations used for Table~\ref{tab:npcifar10_results} consist of block sizes $\{128, 8, 8, 8\}$ for the SpatialRNN of 152 neurons (51k parameters), and block sizes $\{64, 8, 8, 8\}$ for the SpatialRNN of 88 neurons, resulting in 21k parameters. We performed a hyperparameter search over various block size configurations. Table~\ref{tab:npcifar10_sweep} details the specific sizes evaluated, their corresponding parameter counts, and the resulting maximum accuracy (over 3 runs) on the full $T=1000$ sequence length. We bound the matrix norms by $\|\bar{W}_i\|_2 \le 1.7$ and $\|B_i\|_2, \|W_{\psi, i}\|_2 \le 0.2$.

\begin{table}[h]
\centering
\begin{tabular}{lccc}
\toprule
Model & Neurons & Parameters & Max Accuracy \\
\midrule
SpatialRNN [128, 8, 8, 8] & 152 & 50,937 & 57.8 \\
SpatialRNN [64, 8, 8, 8] & 88 & 21,209 & 57.2 \\
SpatialRNN [32, 8, 8, 8] & 56 & 10,953 & 54.1 \\
SpatialRNN [16, 8, 8, 8] & 40 & 6,977 & 50.8 \\
SpatialRNN [8, 8, 8, 8] & 32 & 5,277 & 45.6 \\
\bottomrule
\end{tabular}
\smallskip
\caption{npCIFAR10 task architecture configurations and performance. Maximum accuracy is evaluated at the sequence length $T=1000$ over 3 runs.}
\label{tab:npcifar10_sweep}
\end{table}

\textbf{Compared models.}
To appropriately benchmark the SpatialRNN, we select a strong suite of competitors previously evaluated on the npCIFAR10 task. These include traditional gated architectures (LSTM, GRU) to demonstrate the breakdown of standard gating mechanisms over long sequences, as well as state-of-the-art continuous-time and oscillatory models (e.g., coRNN, LEM, Lipschitz RNN) that specifically target the vanishing gradient problem. To ensure a rigorous and equivalent comparison of intrinsic memory capacity, all evaluated baseline models are restricted to single-layer architectures.

\subsection{Hardware and Compute}

All experiments were implemented in PyTorch and executed on a private workstation equipped with two AMD EPYC 9534 64-core processors (total of 128 physical CPU cores and 256 threads) and 755 GB RAM. Model training and evaluation were performed on CPU resources only; no GPU acceleration or multi-GPU model parallelism was used throughout.  The reported experiments are computationally lightweight relative to large-scale sequence-modeling workloads. The copy-task experiments used synthetically generated sequences and small recurrent models, with the largest reported SpatialRNN architecture containing 6,815 parameters and the largest HORNN architecture containing 17,744 parameters. The psMNIST (resp. npCIFAR) experiments used single-layer SpatialRNNs with up to 36,497 (resp.~50,937) parameters, trained for at most 1000 (resp. 3000) epochs with early stopping over three independent initializations.
The copy-task experiments accounted for approximately 1{,}600 CPU-thread-hours, including sweeps over sequence length and recurrent baselines. The psMNIST experiment required approximately 1{,}600 CPU-thread-hours across the three reported initializations and model-size configurations. The npCIFAR10 experiments required approximately 3{,}200 CPU-thread-hours across the three reported initializations and sequence-length ablations. 
Overall, using up to 20 hardware threads for each experiment type, we estimate that the total compute required for the final reported experiments was approximately two weeks on the reported CPU workstation.

\end{document}